\documentclass{article}
\usepackage[preprint]{neurips_2025}
\ifx\pdftexversion\undefined\else
  \pdfoutput=1
\fi

\title{Distances for Markov chains from sample streams}

\usepackage{float}
\usepackage{times}
\usepackage{wrapfig}
\usepackage{algorithm}
\usepackage{algorithmic}
\usepackage[utf8]{inputenc} % allow utf-8 input
\usepackage[T1]{fontenc}    % use 8-bit T1 fonts
\usepackage{hyperref}       % hyperlinks
\usepackage{url}            % simple URL typesetting
\usepackage{booktabs}       % professional-quality tables
\usepackage{amsfonts}       % blackboard math symbols
\usepackage{nicefrac}       % compact symbols for 1/2, etc.
\usepackage{microtype}      % microtypography
\usepackage{xspace}
\usepackage{amsmath}
\usepackage{subcaption} 

\usepackage{mathtools}  
\usepackage{amsmath}
\usepackage{amssymb}
\usepackage{tabulary}
\usepackage{booktabs}

\usepackage{amsthm}
\usepackage{cleveref}

\usepackage{bm}

\newtheorem{thm}{Theorem}
\newtheorem{lemma}{Lemma}
\newtheorem{proposition}{Proposition}

\usepackage{comment}
\usepackage{selectp}
\usepackage{selectp}
\usepackage{enumitem}

\newcommand{\SOMCOT}{\textbf{\texttt{SOMCOT}}\xspace}

\newcommand{\mupi}{\mu^\pi}
\newcommand{\hy}{\wh{y}}
\newcommand{\hx}{\wh{x}}

\newcommand{\bmu}{\overline{\mu}}

\newcommand{\diag}{\mathrm{diag}}
\newcommand{\dd}{\mathrm{d}}
\newcommand{\nux}{\nu_\X}
\newcommand{\nuy}{\nu_\Y}
\newcommand{\etax}{\eta_\X}
\newcommand{\etay}{\eta_\Y}
\newcommand{\betax}{\beta_\X}
\newcommand{\betay}{\beta_\Y}

\makeatletter

\newcommand{\@nextsubscript}[2]{\expandafter#1\@gobble#2}

\newcommand{\alphax}{%
  \@ifnextchar_%
    {\@nextsubscript\@alphaxWith}%
    {\alpha_{\X}}%
}
\newcommand{\alphay}{%
  \@ifnextchar_%
    {\@nextsubscript\@alphayWith}%
    {\alpha_{\Y}}%
}
\newcommand{\lambdax}{%
  \@ifnextchar_%
    {\@nextsubscript\@lambdaxWith}%
    {\lambda_{\X}}%
}
\newcommand{\lambday}{%
  \@ifnextchar_%
    {\@nextsubscript\@lambdayWith}%
    {\lambda_{\Y}}%
}

\newcommand{\blambdax}{%
  \@ifnextchar_%
    {\@nextsubscript\@lambdaxWith}%
    {\overline{\lambda}_{\X}}%
}
\newcommand{\blambday}{%
  \@ifnextchar_%
    {\@nextsubscript\@lambdayWith}%
    {\overline{\lambda}_{\Y}}%
}

\newcommand{\@alphaxWith}[1]{\alpha_{\X,#1}}
\newcommand{\@alphayWith}[1]{\alpha_{\Y,#1}}
\newcommand{\@lambdaxWith}[1]{\lambda_{\X,#1}}
\newcommand{\@lambdayWith}[1]{\lambda_{\Y,#1}}

\makeatother

\newcommand{\X}{\mathcal{X}}

\newcommand{\PX}{P_\X}
\newcommand{\PY}{P_\Y}
\newcommand{\Y}{\mathcal{Y}}

\newcommand{\LL}{\mathcal{L}}
\newcommand{\GG}{\mathcal{G}}
\newcommand{\BB}{\mathcal{B}}

\newcommand{\CX}{\mathcal{C}_\mathcal{X}}
\newcommand{\CY}{\mathcal{C}_\mathcal{Y}}

\newcommand{\F}{\mathcal{F}}

\newcommand{\real}{\mathbb{R}}

\newcommand{\Sw}{\mathcal{S}}

\newcommand{\OO}{\mathcal{O}}

\newcommand{\II}[1]{\mathbf{1}_{\left\{#1\right\}}}
\newcommand{\PP}[1]{\mathbb{P}\left[#1\right]}

\newcommand{\EEs}[2]{\mathbb{E}_{#2}\left[#1\right]}

\newcommand{\PPpi}[1]{\mathbb{P}_{\pi}\left[#1\right]}

\newcommand{\EEcc}[2]{\mathbb{E}\left[\left.#1\right|#2\right]}

\def\argmin{\mathop{\mbox{ arg\,min}}}
\def\argmax{\mathop{\mbox{ arg\,max}}}
\newcommand{\ra}{\rightarrow}

\newcommand{\siprod}[2]{\langle#1,#2\rangle}
\newcommand{\iprod}[2]{\left\langle#1,#2\right\rangle}

\newcommand{\norm}[1]{\left\|#1\right\|}

\newcommand{\onenorm}[1]{\norm{#1}_1}
\newcommand{\twonorm}[1]{\norm{#1}_2}
\newcommand{\infnorm}[1]{\norm{#1}_\infty}

\newcommand{\ev}[1]{\left\{#1\right\}}
\newcommand{\abs}[1]{\left|#1\right|}

\newcommand{\pa}[1]{\left(#1\right)}
\newcommand{\bpa}[1]{\bigl(#1\bigr)}

\newcommand{\wh}{\widehat}
\newcommand{\wt}{\widetilde}

\newcommand{\transpose}{^\mathsf{\scriptscriptstyle T}}

\usepackage[disable]{todonotes}
\definecolor{PalePurp}{rgb}{0.66,0.57,0.66}
\newcommand{\todoG}[1]{\todo[color=PalePurp!30]{\textbf{Grg:} #1}}

\newcommand{\regret}{\mathrm{regret}}
\newcommand{\regretmax}{\regret^{\max}}
\newcommand{\regretmin}{\regret^{\min}}
\newcommand{\rsym}{r^{\text{sym}}}

\newcommand{\Zw}{\mathcal{Z}}

\usepackage[normalem]{ulem}
\usepackage{enumitem}

\newcommand{\myc}{\partial \CY}
\newcommand{\mxc}{\partial \CX}
\newcommand{\bfc}{\partial \F}

\newcommand{\CCX}{\mathcal{C}_{\X}}
\newcommand{\CCY}{\mathcal{C}_{\Y}}

\newcommand{\MX}{M_{\X}}
\newcommand{\MY}{M_{\Y}}
\newcommand{\MXY}{M_{\X,\Y}}
\author{Sergio Calo\,\,\,\, Anders Jonsson \,\,\,\, Gergely Neu \,\,\,\, Ludovic Schwartz 
\,\,\,\, Javier Segovia-Aguas
 \\
 Universitat Pompeu Fabra, Barcelona, Spain\\
 {\scriptsize{\texttt{\{sergio.calo,anders.jonsson,gergely.neu,ludovic.schwartz,javier.segovia\}@upf.edu}}}
 }
 
\begin{document}

\maketitle

\begin{abstract}
Bisimulation metrics are powerful tools for measuring similarities between stochastic processes, and specifically 
Markov chains. Recent advances have uncovered that bisimulation metrics are, in fact, optimal-transport distances, 
which has enabled the development of fast algorithms for computing such metrics with provable accuracy and 
runtime guarantees. However, these recent methods, as well as all previously known methods, assume full knowledge of the transition 
dynamics. This is often an impractical assumption in most real-world scenarios, where typically only sample trajectories are available.
In this work, we propose a stochastic optimization method that addresses this limitation
and estimates bisimulation metrics based on sample access, without requiring explicit transition models.
Our approach is derived from a new linear programming (LP) formulation of bisimulation metrics, which we solve using a 
stochastic primal-dual optimization method. We provide theoretical guarantees on the sample complexity of the algorithm 
and validate its effectiveness through a series of empirical evaluations.
\end{abstract}

\section{Introduction}

Computing similarity metrics between stochastic processes is an important mathematical problem with numerous 
promising use cases in diverse areas such as mathematical finance, computational neuroscience, biology, and computer 
science. Within machine learning, potential applications include representation learning for dynamical systems and 
reinforcement learning~\citep{ZMCGL21,CP22}, fitting and comparing sequence models \citep{XWMA20,TNL24}
or prediction tasks on graph-structured data \citep{vayer2019graphs,BWW24}. While there exist 
several rigorous frameworks for defining such similarity metrics and studying their 
properties, computing them typically requires full knowledge of the probability law of the processes to compare, which 
is not available in just about any case of practical interest. In this paper, we address this problem by developing 
methods for estimating similarity metrics for a family of stochastic processes, based only on sample streams 
and without requiring any prior information about the underlying process laws.

We focus on a family of similarity metrics known as \emph{bisimulation metrics}, originating from theoretical computer 
science for purposes of formal verification of computer programs~\citep{Park81,Milner89,DGJP99,vBW01}. This notion of process similarity has 
gained popularity within reinforcement learning (RL), where its potential for learning state representations has been 
recognized by the early works of~\citet{givan2003equivalence} and~\citet{FPP04} and the possibility of using it as a 
basis of practical methods 
for representation learning has been explored in a long line of subsequent works~\citep{castro2020scalable,ZMCGL21,CP22,kemertas2022approximate}. Another popular framework 
for studying similarities between structured probability distributions is that of \emph{optimal transport} 
(OT, cf.~\citealp{villani2009optimal}), which has received serious attention within machine learning in the last 
decade, largely owing to the work of~\citet{cuturi2013sinkhorn}. Very recently, 
\citet{Calo_J_N_S_S24} pointed out that bisimulation metrics also fall within the family of OT distances, which not 
only allowed them to connect two distinct areas of mathematics but also import tools from the literature of 
computational optimal transport~\citep{COTFNT} and develop more efficient methods for computing bisimulation metrics. 
We refer to Appendix~A of \citet{Calo_J_N_S_S24} for more historical details on the two extensive lines of literature 
on bisimulation metrics and optimal transport for stochastic processes.

In this paper, we extend this line of work and show that recasting bisimulation metrics as OT distances allows 
not only computational advances, but the development of a rigorous theory for \emph{statistical estimation} of 
similarity metrics between stochastic processes. In particular, we build on the foundations laid down by 
\citet{Calo_J_N_S_S24} and derive a new stochastic optimization algorithm for estimating bisimulation metrics based on 
sample observations only, and provide its complete computational and sample-complexity analysis for finite Markov 
chains. A core technical contribution is a new linear-program formulation of bisimulation metrics, which we solve via 
a stochastic saddle-point optimization method. For two Markov chains with state spaces $\X$ and $\Y$, the 
algorithm is guaranteed to return an $\varepsilon$-accurate estimate of the true similarity metric after 
$\wt{\OO}(\abs{\X}\abs{\Y}(\abs{\X} + 
\abs{\Y})/\varepsilon^2)$ iterations, with each iteration making use of a single sample transition from each of the two 
chains, and costing $\Theta(\abs{\X}^2\abs{\Y}^2)$ computation. This is the first result of its kind: no previous 
methods have successfully addressed this problem either in practice or in theory.

As mentioned above, the problem we study in this paper has been extensively studied in (at least) two major 
research 
communities. Within the optimal-transport community, the problem of computing distances between stochastic processes 
has been studied under the names ``adapted'', ``causal'' or ``bicausal'' optimal transport 
\citep{pflug12,lassalle18,BBLZ17,EP24}. Considering the special case of Markov chains (as we do in the present paper), 
\citet{moulos2021bicausal,OMN21} and \citet{BWW24} proposed approximate dynamic programming algorithms based on the observation that 
computing OT distances between Markov chains can be reduced to a problem of optimal control in Markov decision 
processes. \citet{Calo_J_N_S_S24} developed a novel linear programming framework for computing OT distances
between Markov chains, and showed that such distances are equivalent to bisimulation metrics.
However, previous approaches in this line of work all assume known transition dynamics.

Within the theoretical computer science community, the study of bisimulation metrics progressed quite 
differently: after an initial flurry of foundational works of \citet{DGJP99,vBW01} and \citet{desharnais2002metric}, surprisingly 
few studies have addressed computational matters (one rare example being the work of \citet{CvBW12}).
Several recent works in reinforcement learning aim to learn approximate bisimulation metrics from
sample transitions using deep learning~\citep{castro2020scalable,ZMCGL21,CP22,kemertas2022approximate}.
However, these approximate bisimulation metrics are not well-founded in theory and as a consequence,
do not enjoy the theoretical guarantees of the original metrics.
In contrast, the stochastic-optimization method we develop in this paper is 
firmly rooted in a theoretical understanding of the problem and comes with provable computational and sample-complexity 
guarantees.

The use of stochastic solvers to compute OT distances has been explored in several past works, mostly in the context 
of static optimal transport between probability distributions. A good part of these methods are based on the 
observation that the static OT problem is formulated as a linear program, and the associated 
unconstrained dual optimization problem directly lends itself to numerical optimization. This view has been exploited 
by works like \citet{genevay16}, \citet{ACB17}, and \citet{seguy2017large}, with some rigorous performance guarantees 
provided by \citet{BBB20}. Another line of work makes use of Monte Carlo estimates of OT 
distances \citep{genevay2018learning,fatras2019learning,fatras2021minibatch,mensch2020online}. To our knowledge, the 
idea of computing OT 
distances via stochastic primal-dual methods as we do in the present work has not been explored in this literature, and 
thus our contribution may be of independent interest within this context as well.

The rest of the paper is organized as follows. After formally defining our problem in Section~\ref{sec:prelim}, we 
describe the foundations of our new algorithm and describe it in detail in Section~\ref{sec:main}. We state our 
main theoretical results in Section~\ref{sec:analysis}, where we also outline the main ideas of the analysis. 
We complement these with some empirical studies of the newly proposed method in Section~\ref{sec:experiments}, and 
conclude with a discussion of the results and open problems in Section~\ref{sec:conc}.

\vspace{-5mm}
\paragraph{Notation.} For a finite set $ \Sw $, we use $ \Delta_{\Sw} $ to denote the set of all probability 
distributions over $ S $. For two sets $\X$ and $\Y$, we will often write $\X\Y = \X\times\Y$ to abbreviate their 
direct product. We will denote infinite sequences by $ \bar{x}=(x_0,x_1,\ldots) $ and for any $ n$ the 
corresponding subsequences as $ \bar{x}_n =(x_0,  \ldots, x_n)$.

\section{Preliminaries}\label{sec:prelim}
We study the problem of measuring distances between pairs of finite Markov chains.   Specifically, we 
consider two stationary Markov processes $ M_{\X} = (\X, \PX, \nu_{0,\X}) $ and $ M_{\Y} = (\Y, \PY, \nu_{0,\Y})$, where
\begin{itemize}[topsep=-1mm,itemsep=-.5mm,partopsep=1mm,parsep=1mm,leftmargin=5mm]
	\item $ \X $ and $ \Y $ are the \emph{state spaces} with finite cardinality,
	\item $\PX : \X \rightarrow \Delta_\X$ and $ \PY : \Y \rightarrow\Delta_\Y $ are the \emph{transition kernels} that 
	specify the evolution of states as $ \PX(x'|x) = \PP{X_{t+1}=x'|X_t=x} $ and $\PY(y'|y) = 
\PP{Y_{t+1}=y'|Y_t=y} $ (for all time indices $t$ and state pairs $x,x'$ and $y,y'$),
	\item $ \nu_{0,\X}\in \Delta(\X)$ and $\nu_{0, \Y}\in \Delta(\Y) $ are the \emph{initial-state distributions} which 
	specify the states at time $t=0$ as $X_0 \sim \nu_{0, \X}$ and $ Y_0\sim \nu_{0,y} $.
\end{itemize}
Without loss of generality, we will assume that $\nu_{0,\X}$ and $\nu_{0,\Y}$ are both Dirac measures respectively 
supported on some fixed $x_0$ and $y_0$, and use $ \nu_0 = \nu_{0,\X} \otimes \nu_{0, \Y} $ to denote the joint 
distribution of the pair of initial states $(X_0,Y_0)$ (which is of course a Dirac measure on $x_0,y_0$). 
For each $ n\geq 0 $, the objects above define a
sequence of joint distributions $ \PP{(X_0, X_1, \ldots, X_n) = (x_0, x_1, \ldots, x_n)} $ and $ \PP{(Y_0, Y_1, \ldots,
Y_n) = (y_0, y_1, \ldots, y_n)}$. These distributions in turn define the 
laws of the infinite sequences $\overline{X} = (X_0, X_1, \ldots ) $ and $ \overline{Y} = (Y_0, Y_1, \ldots) $ via 
Kolmogorov's extension theorem. 
With a slight abuse of notation we use $ \MX $ and $ \MY $ to denote the corresponding measures satisfying $
\MX(\bar{x}_n) = \PP{\overline{X}_n = \bar{x}_n} $ and $ \MY(\bar{y}_n) =\PP{\overline{Y}_n = \bar{y}_n}$ for any $
\bar{x}\in \X^{\infty}, \bar{y} \in \Y^{\infty} $ and $ n\geq 0 $.

Our goal is to compute optimal transport distances between infinite-horizon Markov chains. 
To this end, we will suppose access to a \emph{ground cost} (or \emph{ground metric}) $ c:\X\Y
\rightarrow\real^{+} $ that quantifies the (dis-)similarity between each state $x\in\X$ and $y\in\Y$ as $c(x,y)$. For 
any two sequences $ \bar{x} = (x_0,x_1,\ldots)\in \X^{\mathbb{N}} $ and $
\bar{y}=(y_0,y_1,\ldots) \in \Y^{\mathbb{N}} $, we define the discounted total cost
\begin{equation*}
	c_\gamma(\bar{x},\bar{y})=\sum_{t=0}^{\infty} \gamma^t c(x_t,y_t),	
\end{equation*}
where $ \gamma \in(0,1) $ is the \textit{discount factor} (which emphasizes earlier differences between the two 
sequences, and serves to make sure that the distance is well-defined). 
As is usual in the optimal-transport literature, we will define the distance between the two stochastic processes $ 
M_\X $ and $ M_\Y $ by minimizing the expected cost over a suitable class of \textit{couplings} of the 
two joint distributions.

Formally, a coupling of $ \MX $ and $ \MY $ is defined as a stochastic process on the joint space $ \X\times\Y $ whose 
law is defined for all $n$ as $ \MXY(\overline{x}_n,\overline{y}_n)  = \PP{\overline{X}_n = x_n, \overline{Y}_n = y_n}$ 
and satisfies $
\sum_{\overline{y}_n \in \Y^n}\MXY(\overline{x}_n \overline{y}_n) = \MX(\overline{x}_n) $ and $ \sum_{\overline{x}_n 
\in \X^n} \MXY(\overline{x}_n, \overline{y}_n) = \MY(\overline{y}_n) $. We denote the set of all couplings by $ \Pi $, 
and call a coupling $ \MXY \in \Pi $ \textit{bicausal} if and only if it satisfies 
\begin{equation*}
	\sum_{y \in \Y} \MXY(xy|\bar{x}_{n-1}\bar{y}_{n-1})	= \MX(x |\bar{x}_{n-1}) \quad \mbox{and} \quad \sum_{x \in \X} 
\MXY(xy|\bar{x}_{n-1}\bar{y}_{n-1}) = \MY(y |\bar{y}_{n-1}),
\end{equation*}
respectively for all $x$ and $y$, and for all $n$. The set of all bicausal couplings will be denoted by $ \Pi_{bc} $. 
Intuitively, this is the class of couplings that respect the temporal structure
of the Markov chains and only allow the distribution of each state $X_{t+1}$ (resp.~$Y_{t+1}$) to be influenced by the 
past state pairs $\pa{\overline{X}_{t},\overline{Y}_{t}}$. The optimal transport distance between the two Markov chains 
$\MX$ and $\MY$ is then defined as 
\begin{equation}
\label{eq:OT_cost}
d_{\gamma}(\MX, \MY) = \inf_{\pi \in \Pi_{\text{bc}}} \int c_{\gamma}(\overline{X},
	\overline{Y}) \, \dd 	\pi (\overline{X}, \overline{Y}),
\end{equation}
with the dependence on the cost function $c$ suppressed for simplicity of notation. Following the observation made by 
\citet{Calo_J_N_S_S24}, we will frequently refer to this distance as the \emph{bisimulation metric} between $\MX$ and 
$\MY$.

\section{Bisimulation metrics from sample streams}\label{sec:main}
As observed by \citet{Calo_J_N_S_S24}, the bisimulation metric in~\eqref{eq:OT_cost} can be rewritten in terms of 
\emph{occupancy couplings}. The occupancy coupling associated with the bicausal coupling $\pi \in \Pi_{\text{bc}}$ is a 
distribution $\mupi\in\Delta_{\X\Y\X\Y}$ with entries
\[
 \mupi(x,y,x',y') = (1-\gamma) \sum_{t=0}^\infty \gamma^t \PPpi{X_t= x, Y_t = y, X_{t+1} = x', Y_{t+1} = y'},
\]
where $\PPpi{\cdot}$ denotes the probability law induced by the coupling $\pi$. Introducing the notation 
$\iprod{\mu}{c} = \sum_{x,y,x',y'} \mu(x,y,x',y') c(x,y)$, this means that the original optimization 
problem defining the distance can be obviously rewritten as a linear function of $\mupi$ as
\begin{equation}
\label{eq:OT_cost_mu}
d_{\gamma}(\MX, \MY) = \inf_{\pi \in \Pi_{\text{bc}}} \iprod{\mupi}{c}.
\end{equation}
\citet{Calo_J_N_S_S24} identified a set of linear constraints on $\mupi$ that are satisfied if and
only if $\pi \in \Pi_{\text{bc}}$, which has effectively reduced the problem of computing the distance to a linear 
program (LP). This formulation is closely related to the standard LP formulation of optimal control in Markov decision 
processes, where the primal variables are commonly called \emph{occupancy measures} (see, e.g., Chapter 6.9 in 
\citealp{Puterman1994}). As shown by \citet{Calo_J_N_S_S24}, one of the linear constraints satisfied by any valid 
occupancy $\mu$ is the following \emph{flow condition}:
\begin{equation}\label{eq:bellman_flow}
\sum_{x',y'} \mu(x,y, x',y') = \gamma \sum_{\hx,\hy} \mu(\hx,\hy, x,y) + (1-\gamma) \nu_0(x,y) 
\qquad(\forall x,y \in \X\Y).
\end{equation}
Unfortunately, their other constraints explicitly feature the transition kernels $\PX$ and $\PY$, which ultimately 
makes their LP unsuitable as a basis for stochastic optimization. Indeed, optimizing their LP via primal, dual, or 
primal-dual methods would require having at least a generative model of $\PX$ and $\PY$ that allows sampling from 
$\PX(\cdot|x)$ and $\PY(\cdot|y)$ at arbitrary states $x$ and $y$. In practice however, such models are rarely available 
and one has to make do with samples drawn directly from a stream of states generated by the two chains.
We address this limitation by reformulating their linear constraints in a form that eliminates the transition kernels 
$\PX$ and $\PY$ from the constraints, and replaces them with a joint state-next-state distribution from each chain 
that can be sampled from efficiently. In what follows, we first introduce our new LP formulation, and then provide a 
primal-dual stochastic optimization algorithm to approximately solve the resulting optimization problem along with its 
performance guarantees.

\subsection{A new LP formulation of bisimulation metrics}
Our reformulation is based on the following observations. First, notice that any valid 
occupancy coupling has to arise as a coupling of the marginal occupancy measures of the two chains $\MX$ and $\MY$, 
defined respectively for each $x,x'$ and $y,y'$ as
\begin{align*}
 \nux(x,x') &= \textstyle{(1-\gamma) \sum_{t=0}^\infty \gamma^t \PP{X_t = x, X_{t+1} = x'}}, \\
 \nuy(y,y') &= \textstyle{(1-\gamma) \sum_{t=0}^\infty \gamma^t \PP{Y_t = y, Y_{t+1} = y'}}.
\end{align*}
Indeed, valid occupancy couplings respectively satisfy the \emph{coupling condition}
\begin{equation}\label{eq:coupling}
 \textstyle{\sum_{y,y'} \mupi(x,x',y,y') = \nux(x,x') \qquad \mbox{and} \qquad \sum_{x,x'} \mupi(x,x',y,y') = 
\nuy(y,y')}
\end{equation}
for all $x,x'$ and $y,y'$. Second, the conditional occupancies induced by a bicausal coupling $\pi$ satisfy
\begin{equation*}%\label{eq:causality}
 \mupi(x',y|x) = \PX(x'|x) \mupi(y|x) \qquad \mbox{and} \qquad \mupi(x,y'|y) = \PY(y'|y) \mupi(x|y),
\end{equation*}
due to the requirement of \emph{causality} that the conditional law of $Y_t$ given $X_t$ (resp.~$X_t$ given $Y_t$) 
should be independent of the future state $X_{t+1}$ (resp.~$Y_{t+1}$). By multiplying both sides of these equations by 
$\sum_{x'}\nux(x,x')$ and $\sum_{y'} \nuy(y,y')$, we obtain
\begin{equation}\label{eq:causality+coupling}
 \textstyle{\sum_{y'} \mupi(x,x',y,y') = \nux(x,x') \mupi(y|x) \quad \mbox{and} \quad \sum_{x'} \mupi(x,x',y,y') = 
\nuy(y,y') \mupi(x|y)}.
\end{equation}
Summing both sides for all $y$ and $x$ respectively recovers the coupling conditions of Equation~\eqref{eq:coupling}. 
In this sense, both the causality and coupling conditions can be recovered by the single set of 
equations~\eqref{eq:causality+coupling}.
% , which we accordingly call \emph{bicausal coupling condition}. \grg{better name needed!} 
The following key result shows that, together with the flow constraints of Equation~\eqref{eq:bellman_flow}, this 
system of equations provides a complete characterization of occupancy couplings.
\begin{proposition}
\label{prop:causality_constraints}
The distribution $\mu$ is the induced occupancy coupling of a bicausal coupling $\pi \in \Pi_{\text{bc}}$ if and only if there exist 
$\lambdax\in\real^{\Y\X}_+$ and $\lambday\in\real^{\X\Y}_+$ such that the following equations hold:
\begin{align}
\label{eq:bellman_flow_2}
 \sum_{x',y'} \mu(x,y, x',y') &= \gamma \sum_{\hx,\hy} \mu(\hx,\hy, x,y) + (1-\gamma) \nu_0(x,y) 
\qquad(\forall x,y \in \X\Y)\\
\label{eq:X_causal}
 \sum_{y'} \mu(x,y,x',y') &= \nux(x,x')\lambdax(y|x) \qquad\qquad\qquad\qquad\quad\;\,(\forall x,x',y \in \X\X\Y)\\
 \label{eq:Y_causal}
 \sum_{x'} \mu(x,y,x',y') &= \nuy(y,y') \lambday(x|y) \qquad\qquad\qquad\qquad\quad\;\;(\forall x,y,y' \in \X\Y\Y).
\end{align}
Furthermore, if the equations are satisfied for some $\mu$, $\lambdax$ and $\lambday$, we also have $\sum_{y} 
\lambdax(y|x) = 1$ and $\sum_{x} \lambday(x|y) = 1$ for all $x$ and $y$. 
\end{proposition}
Thus, the set of 
equations~\eqref{eq:bellman_flow_2}--\eqref{eq:Y_causal} uniquely identifies the complete set of occupancy couplings. 
In 
particular, whenever the constraints are satisfied for some $\mu$, there exists a bicausal coupling $\pi$ 
inducing $\mu$ as its occupancy coupling, and conversely all occupancy couplings satisfy the above equations. 
Further important side results can be read out from the proof, provided in Appendix~\ref{app:lp}.

\subsection{A stochastic primal-dual method}
An immediate consequence of \Cref{prop:causality_constraints} is that the OT distance between $\MX$ and $\MY$ can be 
written as the solution of the minimization problem of Equation~\eqref{eq:OT_cost_mu} with respect to $\mu^\pi$ 
as the optimization variable, subject to the constraints 
\eqref{eq:bellman_flow_2}--\eqref{eq:Y_causal}. Equivalently, it can be written as a saddle point of the associated 
Lagrangian defined as 
\begin{align}
\nonumber \!\LL(\mu, \lambda; \alpha, V) =& \hspace*{-2pt} \sum_{xyx'y'} \hspace*{-2pt} \mu(x,y,x',y') \pa{c(x,y) + 
\alphax(x,x',y) + \alphay(x,y,y') + \gamma V(x',y') - V(x,y)} \\
&
\nonumber -\sum_{xx'y} \nux(x,x') \lambdax(y|x) \alphax(x,x',y)
-\sum_{xyy'} \nuy(y,y') \lambday(x|y) \alphay(x,y,y') 
\\
&+ (1-\gamma) \sum_{xy} \nu_0(x,y) V(x,y), \label{eq:Lagrangian_LP}
\end{align}
where $\alphax\in\real^{\X\X\Y}$ and $\alphay\in\real^{\X\Y\Y}$ are the Lagrange multipliers 
associated with constraints~\eqref{eq:X_causal} and~\eqref{eq:Y_causal}, and $V\in\real^{\X\Y}$ are the 
multipliers for the flow constraint~\eqref{eq:bellman_flow_2}. Indeed, by the Lagrange multiplier theorem, the optimal 
value of the original LP can be written as $d_\gamma(\MX,\MY) = \min_{\mu,\lambda}\max_{\alpha,V} \LL(\mu, \lambda;  
\alpha, V)$. Importantly, the gradients of the Lagrangian with respect to dual variables $\alphax$ and $\alphay$ can be 
written as expectations with respect to the occupancy measures $\nux$ and $\nuy$, which suggests that the objective may 
be amenable to stochastic optimization given sample access to these distributions.

\begin{wrapfigure}{r}{0.5\textwidth}
\vspace{-.7cm}
\begin{minipage}{0.5\textwidth}
\begin{algorithm}[H]
\caption{\SOMCOT}\label{alg:main}
\textbf{Input: } $c$, $\eta$, $\beta$, $\gamma$, $K$\\
\textbf{Initialize: }  $\mu_1 = \mathcal{U}(\X\Y\X\Y)$, $\lambdax(\cdot|x) = \mathcal{U}(\Y)$ for all $x$, 
$\lambday(\cdot|y) = \mathcal{U}(\X)$ for all $y$, $\alpha=0$, $V=0$.
\\
\textbf{For $k=1,2,\dots,K$:}
\\
\,\textbullet  \,Sample $X_k,X_k'\sim\nux$ and $Y_k,Y_k'\sim\nuy$,\\
\,\textbullet  \,compute gradient estimators via Eqs.~\eqref{eq:mugrad}--\eqref{eq:Vgrad},\\
\,\textbullet  \,update primal parameters via Eqs.~\eqref{eq:mu_update}--\eqref{eq:lambday_update},\\
\,\textbullet  \,update dual parameters via Eqs.~\eqref{eq:alphax_update}--\eqref{eq:V_update}.
\\
\textbf{Output:} $\bmu_K = \frac 1K \sum_{k=1}^K \mu_k$.
\label{alg:short}
\end{algorithm}
\end{minipage}
\vspace{-.4cm}
\end{wrapfigure}

Inspired by this observation, we propose a primal-dual stochastic optimization algorithm that aims to approximate the 
saddle point of the Lagrangian. In particular, we will suppose that we have sampling access to the occupancy measures 
$\nux$ and $\nuy$ and use these samples to construct stochastic gradient estimators for incrementally updating the 
primal and dual variables via variants of stochastic gradient descent-ascent.
Concretely, the algorithm proceeds in a sequence of iterations $k=1,2,\dots,K$, updating the primal variables $\mu_k$, 
$\lambdax_k$ and $\lambday_k$ via stochastic mirror descent (SMD) with entropic regularization and the dual 
variables $V_k$, $\alphax_k$ and $\alphay_k$ via stochastic gradient ascent (SGA). We describe the gradient-estimation 
procedures and the update rules below, and provide a high-level pseudocode as Algorithm~\ref{alg:short}. For brevity, 
we will refer to the algorithm as \SOMCOT, for Stochastic Optimization for Markov Chain Optimal Transport.
Further details about the derivation of \SOMCOT and a more detailed pseudocode can be found in 
Appendix~\ref{app:alg_details}.

\paragraph{The gradient estimators.}
For constructing the gradient estimators needed for the updates, we first sample transitions $(X_{k},X_k') 
\sim \nux$ and $(Y_{k},Y_k') \sim \nuy$ from the marginal occupancy measures of $\MX$ and $\MY$, and let $\F_k$ denote 
the record of all transition data drawn until the end of round $k$. The primal updates are defined in 
terms of the following gradient estimates:
\begin{align}
 g_{k, \mu}(x,y,x',y') &= c(x,y) - \alphax_k(x, x',y) - \alphay_k(x,y, y') + \gamma V_k(x',y') - V_k(x,y)
 \label{eq:mugrad}
 \\
 \wt{g}_{k, \lambdax}(y|x) &= \II{ X_k,X'_k = x,x'} \alphax_k(x, x' ,y)
 \label{eq:lambdaxgrad}
 \\
 \wt{g}_{k, \lambday}(x|y) &= \II{ Y_k,Y'_k= y,y'} \alphay_k(x,y,y').
 \label{eq:lambdaygrad}
\end{align}
Clearly, we have $g_{k,\mu}= \nabla_\mu \LL(\mu_k,\lambda_k;\alpha_k,V_k)$. Furthermore, it is easy to check that 
$\EEcc{\wt{g}_{k, \lambdax}}{\F_{k-1}} = \nabla_{\lambdax} \LL(\mu_k,\lambda_k;\alpha_k,V_k)$ and $\EEcc{\wt{g}_{k, 
\lambday}}{\F_{k-1}} = \nabla_{\lambday} \LL(\mu_k,\lambda_k;\alpha_k,V_k)$. Similarly, we can define the gradient 
estimates for the dual variables as
\begin{align}
\wt{g}_{k, \alphax}(x,x',y) &= \sum_{y'} \mu_k(x,y,x',y') - \II{ X_k,X'_k = x,x'} \lambdax_k(y|x)
\label{eq:alphaxgrad}
\\
\wt{g}_{k, \alphay}(x,y,y') &= \sum_{x'} \mu_k(x,y, x',y') - \II{  Y_k,Y'_k= y,y'} \lambday_k(x|y)
\label{eq:alphaygrad}
\\
g_{k, V}(x,y) &= \sum_{x' y'} \mu_k(x,y, x',y') - (1- \gamma) \nu_0(x,y) - \gamma \sum_{\hx,\hy} \mu_k(\hx,\hy, x,y),
\label{eq:Vgrad}
\end{align}
which are again easily seen to satisfy $\EEcc{\wt{g}_{k, \alphax}}{\F_{k-1}} = \nabla_{\alphax} 
\LL(\mu_k,\lambda_k;\alpha_k,V_k)$, $\EEcc{\wt{g}_{k, \alphay}}{\F_{k-1}} = \nabla_{\alphay} 
\LL(\mu_k,\lambda_k;\alpha_k,V_k)$ and $g_{k, V} = \nabla_{V} \LL(\mu_k,\lambda_k;\alpha_k,V_k)$.

\paragraph{The update rules.}
The primal variables are updated via stochastic mirror descent with appropriately chosen entropic regularization 
functions. For $\mu$, the updates are given as
\begin{equation}
\label{eq:mu_update}
	\mu_{k+1}(x,y,x',y') = \frac{\mu_k(x,y, x',y') \exp( - \eta g_{k, \mu}(x,y,x',y'))}{\sum_{\hx\hy \hx'\hy'}
	\mu_{k}(\hx,\hy, \hx',\hy') \exp( - \eta g_{k, \mu}(\hx,\hy, \hx',\hy'))},
\end{equation}
with $\eta > 0$ being a stepsize parameter, and the $\lambdax$ variables are updated as
\begin{align}
\label{eq:lambdax_update}
	\lambdax_{k+1}(y|x) &= \frac{\lambdax_k(y|x)\exp(- \etax \wt{g}_{k, \lambdax}(y|x))}{\sum_{\hy}
	\lambdax_{k}(\hy| x) \exp (- \etax \wt{g}_{k,\lambdax}(\hy |x))},
	\\
\label{eq:lambday_update}
	\lambday_{k+1}(x|y) &= \frac{\lambday_k(x|y)\exp(- \etay \wt{g}_{k, \lambday}(x|y))}{\sum_{\hx}
	\lambday_{k}(\hx| y) \exp (- \etay \wt{g}_{k,\lambday}(\hx|y))},
\end{align}
with respective stepsize parameters $\etax,\etay>0$. Note that due to the design of the gradient estimators, each 
iteration only needs to update these variables locally at $\lambdax(\cdot|X_k)$ and $\lambday(\cdot|Y_k)$ at the 
sampled states $X_k$ and $Y_k$. For the dual variables, we define $ \Pi_{\mathcal{D}} $ as the orthogonal projection 
operator onto a convex set $\mathcal{D}$, and implement the following projected stochastic gradient ascent updates:
\begin{align}
\label{eq:alphax_update}
\alphax_{k+1} &= \Pi_{\mathcal{D}_{\alpha}}\left[\alphax_{k} - \betax \wt{g}_{k, \alphax}\right]	,
\\
\alphay_{k+1} &= \Pi_{\mathcal{D}_{\alpha}}\left[\alphay_{k} - \betay \wt{g}_{k, \alphay}\right]	,
\\
\label{eq:V_update}
	V_{k+1} &= \Pi_{\mathcal{D}_{V}}\left[V_k - \beta g_{k,V}\right],
\end{align}
where $\mathcal{D}_{V} = \mathcal{B}^{\infty}(0,\frac{2}{1- \gamma})$ and $\mathcal{D}_{\alpha} = 
\mathcal{B}^{\infty}(0,
\frac{6}{1- \gamma})$ are the projection domains for each 
variable, and $\beta,\betax,\betay>0$ are the stepsize parameters.

\paragraph{The output.} The algorithm terminates after $K$ rounds, and produces the average of the primal iterates 
$\bmu_K = \frac{1}{K} \sum_{k=1}^K \mu_k$ as output. From this, an estimate of the distance can be computed as 
$\wh{d}_\gamma(\MX,\MY) = \iprod{\bmu_K}{c}$. Averaging the output variables is motivated by the design of \SOMCOT
as a primal-dual method and its theoretical analysis, and it also helps stabilize the quality of the solution. Indeed, 
primal-dual methods are prone to instability and oscillations, which are smoothed out very effectively by averaging. We 
discuss the role of this step and other practical improvements to the algorithm in Section~\ref{sec:experiments} below, 
and provide further comments on the potential usefulness of other side products computed by \SOMCOT for downstream 
tasks.

\section{Analysis}\label{sec:analysis}
The following theorem is our main theoretical result about the performance of our algorithm. 
\begin{thm}
\label{thm:PAC_bound_main}
Suppose that $\infnorm{c}\le 1$. Let $ \bmu_K = \frac{1}{K} \sum_{k=1}^{K} \mu_k $ be the output of \SOMCOT{}, let 
$\mu^*$ be the optimal 
occupancy coupling achieving the minimum in Equation~\eqref{eq:OT_cost_mu}, and set the learning rates as 
\[
\eta = \sqrt{\frac{\log(\abs{\X}^2\abs{\Y}^2)(1- \gamma)^2}{K}}, \,\,\etax = \sqrt{\frac{\abs{\X}\log \abs{\Y}(1- 
\gamma)^2}{K}}, \,\,\etay = \sqrt{\frac{\abs{\Y}\log \abs{\X}(1- \gamma)^2}{K}},
\]
\[
\betax = \sqrt{\frac{\abs{\X}^2\abs{\Y}}{(1- \gamma)^2K}}, \quad \betay = \sqrt{\frac{\abs{\X}\abs{\Y}^2}{(1- 
\gamma)^2K}}, \quad \beta = \sqrt{\frac{\abs{\X}\abs{\Y}}{(1- \gamma)^2K}}.
\]
Then, the following bound is satisfied with probability at least	$ 1- \delta $:
\begin{equation*}
	\abs{\siprod{\bmu_K - \mu^{*}}{c}} = \mathcal{O}\left(\frac{1}{\sqrt{K}(1- \gamma)}\left( 
\sqrt{\abs{\X}\abs{\Y}\pa{\abs{\X} + \abs{\Y}}}
			+ \sqrt{\pa{\abs{\X} + \abs{\Y}} \log
	\left( 1/\delta \right)} \right)\right).
\end{equation*}
Equivalently, for any $ \varepsilon>0 $, the output satisfies $|\siprod{\bmu_K - \mu^{*}}{c}| \leq \varepsilon$
with probability at least $ 1- \delta $ if the number of iterations is at least $ K \geq K_0 = \mathcal{O}
\left( \frac{\abs{\X}\abs{\Y}\pa{\abs{\X}+\abs{\Y}} + \pa{\abs{\X}+\abs{\Y}}\log \left( 1/\delta \right)}{(1- \gamma)^2
\varepsilon^2} \right)$. 
\end{thm}
A perhaps surprising feature of the sample-complexity guarantee is that it scales with the state spaces as 
$\abs{\X}\abs{\Y}\pa{\abs{\X} + \abs{\Y}}$ instead of the full dimensionality of the decision variables, 
$\abs{\X}^2\abs{\Y}^2$. Note however that each iteration has a computational cost scaling with this full dimensionality. 
The scaling in terms of $\varepsilon$ is optimal up to logarithmic factors, as can be deduced from well-known lower 
bounds for the static OT problem (see, e.g., \citealt{KTM20}). Finally, we note that the big-O notation only hides 
numerical constants, and the bound features no problem-dependent factors whatsoever.

We provide the main idea of the proof below, and relegate the full analysis to Appendix~\ref{app:analysis}.
The main technical idea is to relate the estimated transport cost $ \iprod{\mu_K }{c}$  to the true optimal transport 
cost via the analysis of the \emph{duality gap} associated with the sequence of iterates computed by the algorithm. The 
duality gap $\GG_K(\mu^*,\lambda^*;\alpha^*,V^*)$ against a set of comparator points 
$(\mu^*,\lambda^*;\alpha^*,V^*)$ satisfies
\begin{align}
&\GG_K(\mu^*,\lambda^*;\alpha^*,V^*) = \frac{1}{K} \sum_{k=1}^K \pa{\LL(\mu_k,\lambda_k;\alpha^*,V^*) - 
\LL(\mu^*,\lambda^*;\alpha_k,V_k)} \label{eq:duality_gap}.
 \end{align}
As is standard for analysis of primal-dual methods, the duality gap can be decomposed into the sum of the 
\emph{regrets} of the minimizing player controlling $\mu$ and $\lambda$, and the maximizing player controlling $\alpha$ 
and $V$, which can be controlled using the well-established of online learning \citep{CL06,Ora19}.
For the analysis, we will pick the 
comparator points as follows. For the primal variables, we let $\mu^*$ be the occupancy coupling achieving the minimum 
in Equation~\eqref{eq:OT_cost_mu} and let the $\lambda^*$ variables be the conditional distributions of $Y|X$ and $X|Y$ 
under the joint distribution $\mu^*$. For the dual variables, we choose
\[
 (\alpha^*,V^*) = \argmax_{\alpha \in \mathcal{D}_\alpha, V\in \mathcal{D}_V} \frac{1}{K}\sum_{k=1}^K 
\LL(\mu_k,\lambda_k;\alpha,V).
\]
Under these choices, the error can be upper bounded as follows.
\begin{lemma} \label{lem:main_error}
$\abs{\siprod{ \bmu_K- \mu^{*}}{c}} \le \GG_K(\mu^*,\lambda^*,\alpha^*,V^*)$.
\end{lemma}
The proof of this lemma makes up the bulk of the analysis, and is thus relegated to Appendix~\ref{app:main_error}. It 
then remains to upper-bound the regrets of the two sets of players, which is routine work that we execute in 
Appendix~\ref{app:regret_bounds}.

\section{Experiments}\label{sec:experiments}
We performed a suite of numerical experiments to study the empirical behavior of our newly proposed algorithm, as 
well as to illustrate some potential applications that are enabled by our method. Due to space restrictions, we only 
show a small portion of the results here, and refer the reader to Appendix \ref{sec:additional_exp} for additional 
results and implementation details (most notably a detailed discussion on hyperparameter-tuning).

Several of the experiments are conducted with a family of processes we call \emph{block Markov chains}, motivated by 
the framework of block Markov decision processes (or block MDPs, \citealt{du2019block}). This framework is commonly 
studied in the context of representation learning for reinforcement learning, where a standard postulate is that the 
dynamics of the environment are governed by a simple latent structure. Block Markov chains formalize this setting by 
assuming the existence of a latent Markov chain with a small discrete state space, with each latent state generating a 
unique set of observations. Formally, we emulate the block structure by fixing a low-dimensional chain 
$\MX$ and another chain $\MY$ that is a copy of $\MX$ up to an additional irrelevant noise variable. In our 
experiments, we let $\MX$ be a uniform random walk on the state space $\X = \ev{1,2,\dots,n}$ and $\MY$ is a Markov 
chain on the state space $\X\times\ev{1,2,\dots,B}$, with the value in $\ev{1,2,\dots,B}$ generated uniformly at 
random. In all experiments, we use a sparse cost function that only allows to clearly distinguish between states $x=1$ 
and $x=n$, and treats all other states as identical.

\paragraph{Representation learning.}
Within the family of block Markov processes, the task of representation learning is equivalent to finding the mapping 
between the latent states and the observations and vice versa. Our method is very well suited for this task, thanks to 
the following curious observation. Besides the estimated coupling $\bmu_K$ and the associated cost, the algorithm 
outputs other values that are potentially useful. Among these, the variables $\blambdax = \frac{1}{K} \sum_{k=1}^K 
\lambdax_{k}$ and $\blambday = \frac{1}{K} \sum_{k=1}^K \lambday_{k}$ are particularly interesting for purposes of 
representation, as these conditional distributions can be interpreted as an \emph{encoder-decoder} pair, with 
$\lambdax(\cdot|x)$ and $\lambday(\cdot|y)$ giving the respective conditional distributions of $Y|X=x$ and $X|Y=y$ 
under the estimate of the optimal coupling. To illustrate the potential usefulness of these maps, we conducted a set of 
experiments on block Markov chains with parameters $n=10$ and $B=5$,
and show the encoder-decoder pairs computed by \SOMCOT in Figure~\ref{fig:encoder-decoder}. Notably, the algorithm 
does not make use of any prior structural knowledge of the environment: each individual state $y$ is treated as a 
separate state. Despite this and the very limited information revealed by the cost function, a block structure is 
clearly identified by \SOMCOT after sufficiently many samples.

\begin{figure}
\centering
 \includegraphics[width = .325\textwidth]{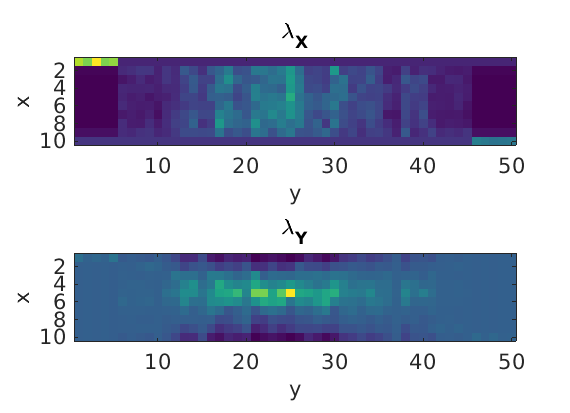}
 \includegraphics[width = .325\textwidth]{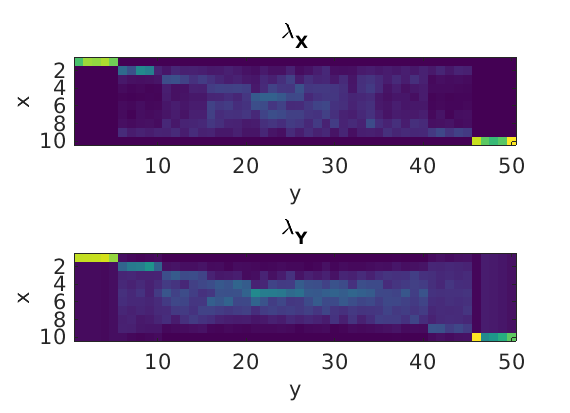}
 \includegraphics[width = .325\textwidth]{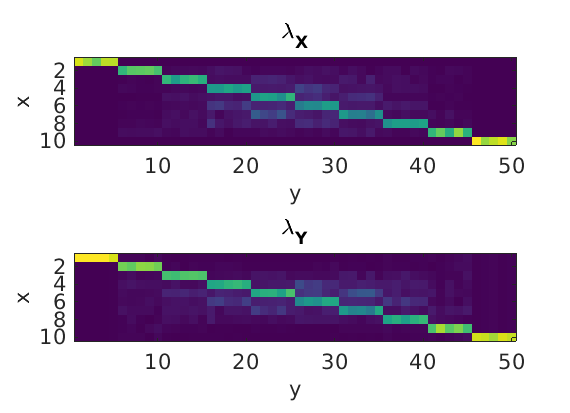}
 \caption{Encoder-decoder maps learned by the algorithm in a block Markov chain example ($n = 10$, $B=5$) for sample 
sizes $1000$, $10000$ and $100000$.}\label{fig:encoder-decoder}
\end{figure}

\paragraph{Model selection.} Another important use case is identifying the hidden dynamics underlying realizations of 
stochastic processes. We model this scenario in two experiments. In the first one, we generate a block-structured 
Markov chain $\MY^*$, and a set of low-dimensional Markov chains $\MX$ with different transition kernels parameterized 
by $\theta\in[0,1]$. This set contains the true model $\MX^*$ underlying $\MY^*$, corresponding to $\theta=0.5$. 
We compute estimates of the distances between $\MY$ and all the candidates of the model class by running \SOMCOT for 
various sample sizes, and show the results on Figure~\ref{fig:blockMC}. Notably, the distance achieves its minimum for 
the true model, and increases as $\theta$ is further separated from its true value.

\begin{figure}
    \centering
    \begin{subfigure}[t]{0.48\textwidth}
        \centering
        \includegraphics[width=0.95\linewidth]{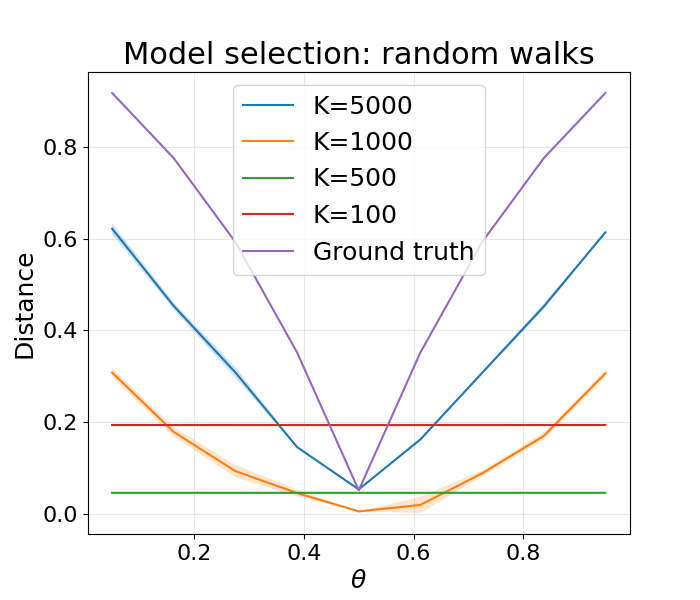}
	\caption{Distances between the true dynamics ($\theta=0.5$) and the different models in the model class, for different values of the parameter $\theta$ and the number of iterations $K$.}
	\label{fig:blockMC}
    \end{subfigure}
    \hfill
    \begin{subfigure}[t]{0.48\textwidth}
        \centering
        \includegraphics[width=0.95\linewidth]{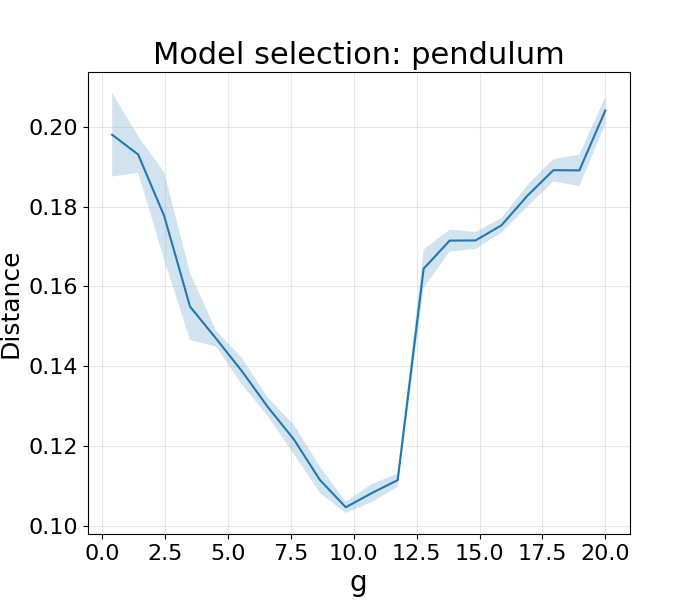}
	\caption{Distances between the true dynamics ($g=9.8$) and the different models in the model class, for different values of the gravity parameter $g$.}
	\label{fig:pendulum}
    \end{subfigure}
	\caption{Model selection results for random walks and the pendulum environment}
    \label{fig:model_select}
\end{figure}

We also conduct a second experiment on model selection in \emph{continuous} state spaces. To this end, we consider the 
classic control environment \textit{Pendulum-v1} from Gymnasium \citep{towers2024gymnasium}. We begin by training a 
near-optimal policy using the DDPG algorithm \citep{Lillicrap2015ContinuousCW} and fix this policy to induce a Markov 
chain over the environment. The continuous state variables of the pendulum are then discretized to $n$ bins each.
We instantiate several copies of the environment by varying a hyperparameter: the acceleration constant $g$, which by 
default is set to $g=9.8$. 
The learning task we consider is to identify which of the class of candidate models 
best explains the unknown true dynamics corresponding to the default choice of $g$. Figure \ref{fig:pendulum} shows the 
results obtained. Again, we can observe that the distance achieves its minimum for the true model, and increases as $g$ 
departs from its true value. Notably, due to discretization of the state space, the observations are not Markovian, yet 
the results clearly indicate that \SOMCOT is still able to produce meaningful distance estimates, thus illustrating the 
potential of this methodology for general representation learning tasks.

\section{Discussion}\label{sec:conc}
In this work, we have explored the use of stochastic methods to compute distances between Markov chains.
This is still a largely unexplored field, and we believe the results presented here open the door to
many interesting advances. We outline some of these future research directions we consider to be the most promising.

Most importantly, it remains unclear how to properly scale our algorithm to larger problems with potentially infinite 
state spaces. While we believe that our bounds cannot be improved significantly in the case of finite state spaces, 
addressing infinite state spaces should be possible under appropriate structural assumptions. One may take direct 
inspiration from the OT literature to extend our approach to these settings. For instance, parametrizing the dual 
variables via kernels or neural networks has been shown to be an effective approach to solve static OT problems 
(cf.~\citealt{genevay16, seguy2017large}), and extending this idea to our setting is straightforward. The real 
challenge seems to be approximating the primal variables, which correspond to (conditional) probability distributions, 
which are not straightforward to parametrize via modern architectures (at least as long as one is interested in 
theoretically sound methods). We leave the investigation of this very interesting question open for future work.

Among all applications of optimal transport for Markov chains, its use in representation learning for RL is 
particularly interesting to us. Many previous works on this domain have highlighted the potential of bisimulation 
metrics for learning state abstractions, but all theoretically sound previous methods for computing such distances 
required full knowledge of the transition kernels. Removing this need brings us closer to realizing this potential.
The experiments presented here demonstrate the effectiveness of bisimulation metrics in capturing symmetries and latent 
dynamics of Markov chains directly from sampled trajectories, both in random walks and discretized classical control 
environments. Incorporating function approximation along the lines mentioned above could significantly enhance these 
applications.

Besides the already-mentioned interpretation of the variables $\lambda_x$ and $\lambda_y$ as encoder-decoder maps, 
there are other side products of \SOMCOT that can prove useful for representation learning. Most notably, the 
optimal dual variables $\alphax$ and $\alphay$ correspond to the derivatives of the distance with 
respect to the state-transition distributions $\nux$ and $\nuy$, which is a fact that can prove extremely useful for 
the development of practical methods. Indeed, notice that these distributions themselves are differentiable with 
respect to the transition kernels, which altogether allows one to backpropagate through the OT distance as a loss 
function in representation learning tasks. Successful implementation of this idea may lead to strong
theoretically sound alternatives to empirically successful methods such as MuZero \citep{schrittwieser2020}. This latter 
method uses a loss function remarkably similar to our OT distance, albeit with some limitations that disallow its 
application to stochastic environments (cf.~\citealt{jiang2024}). Once again, we leave this direction for future work.

\paragraph{Acknowledgements.}
The authors wish to thank Csaba Szepesv\'ari for thought-provoking discussions during the preparation of this 
manuscript. This project has received funding from the European Research Council (ERC), under the European
Union’s Horizon 2020 research and innovation programme (Grant agreement No.~950180).
This work has been co-funded by MICIU/AEI/UE-PID2023-147145NB-I00, AGAUR SGR and
MCIN/AEI/10.13039/501100011033 under the Maria de Maeztu Units of Excellence Programme
(CEX2021-001195-M).

\newpage
\bibliographystyle{plainnat}
\bibliography{ngbib}

\newpage

\newpage

\appendix
\section{Equivalence of the LP formulation and the bisimulation metric}
\label{app:lp}
In this section we prove Proposition~\ref{prop:causality_constraints}, which together with Equation~\eqref{eq:OT_cost_mu} implies that the novel LP
formulation is equivalent to Equation~\eqref{eq:OT_cost} for computing the bisimulation metric.
Our proof uses the linear programming formulation of \citet{Calo_J_N_S_S24} as a starting point.
Concretely, \citet{Calo_J_N_S_S24} prove that $\mu$ is the induced occupancy coupling of a
bicausal coupling $\pi \in \Pi_{\text{bc}}$ if and only if $\mu$ satisfies the following set of constraints:
\begin{align}
\label{eq:bellman_flow_3}
\sum_{x',y'} \mu(x,y,x',y') = \gamma \sum_{x',y'} \mu(x',y',x,y) + (1\!-\!\gamma) \nu_0(x,y) \quad 
(\forall x,y\in\X\!\!\times\!\!\Y),
\\
\label{eq:X_marginal}
\sum_{y'} \mu(x,y,x',y') = \sum_{x'',y'} \mu(x,y,x'',y') \PX(x'|x) 
\qquad (\forall x,y,x'\in\X\!\!\times\!\!\Y\!\!\times\!\!\X),
\\
\label{eq:Y_marginal}
\sum_{x'} \mu(x,y,x',y') = \sum_{x',y''} \mu(x,y,x',y'') \PY(y'|y)
\qquad (\forall x,y,y'\in\X\!\!\times\!\!\Y\!\!\times\!\!\Y).
\end{align}
Importantly, the above constraints provide a complete characterization of occupancy couplings: not only do occupancy 
couplings satisfy all equations, but \emph{any} $\mu$ satisfying the three linear systems of equations above is a valid 
occupancy coupling (cf.~Lemma~1 in \citealp{Calo_J_N_S_S24}).

To prove the proposition it is sufficient to show that $\mu$ satisfies
Equations~\eqref{eq:bellman_flow_2}--\eqref{eq:Y_causal} if and only if it satisfies Equations~\eqref{eq:bellman_flow_3}--\eqref{eq:Y_marginal}.
Equation~\eqref{eq:bellman_flow_2} is identical to Equation~\eqref{eq:bellman_flow_3},
and thus we are left with showing that Equations~\eqref{eq:X_causal} 
and~\eqref{eq:Y_causal} are equivalent to Equations~\eqref{eq:X_marginal}  and~\eqref{eq:Y_marginal},
respectively. We will show the first of these claims, and note that the second claim will follow by 
symmetry. Within the proof, we will repeatedly make use of the shorthand notation $\nux(x) = \sum_{x'} \nux(x,x')$ and 
the easy-to-see fact that $\PX(x'|x) = \nux(x,x')/\nux(x)$.

First, let us assume that $ \mu \in \real^{\X\times\Y\times\X\times\Y} $ satisfies Equation~\eqref{eq:X_marginal}. Then, it 
is easy to check that Equation~\eqref{eq:X_causal} is satisfied with the choice $\lambdax(y|x)= 
\sum_{x',y'}\mu(x,y,x',y')/\nux(x)$, making use of the relation between $\nux$ and $\PX$ stated 
above. Conversely, assume that $ \mu, \lambdax $ satisfy Equation~\eqref{eq:X_causal}. Summing both sides over $x'$ gives
$\sum_{x',y'} \mu(x,y,x',y') = \nux(x) \lambdax(y|x)$, 
which can be plugged back into Equation~\eqref{eq:X_causal} to obtain
\begin{align*}
	\sum_{y'} \mu(x,y, x',y') & = \nux(x,x') \lambdax(y|x) = \PX(x'|x)\nux(x)\lambdax(y|x) = \PX(x'|x) 
\sum_{x',y'} \mu(x,y,x'',y'),
\end{align*}
thus confirming that Equation~\eqref{eq:X_marginal} is indeed satisfied.

For the last part, let us define $\mu_\X$ as $\mu_\X(x,x')=\sum_{y,y'}\mu(x,y,x',y')$ for each $(x,x')$ whenever 
Equations~\eqref{eq:bellman_flow_2}--\eqref{eq:Y_causal} are satisfied. As per the above argument, 
Equations~\eqref{eq:bellman_flow} and \eqref{eq:X_marginal} are also satisfied, and thus summing both equations over 
$y$ yields
\begin{align*}
\sum_{x'} \mu_\X(x,x') &= \gamma \sum_{x'} \mu_\X(x',x) + (1-\gamma) \nu_{0,\X}(x),\\
\mu_\X(x,x') &= \PX(x'|x) \sum_{x''} \mu_\X(x,x'').
\end{align*}
A standard argument (provided as Lemma~\ref{lemma:marginal_occupancy} in Appendix~\ref{app:tech}) shows that the unique solution to this system 
of equations is equal to the marginal occupancy measure $\nux$. This implies
\[
\sum_y \lambdax(y|x) = \sum_y \frac{\sum_{x',y'}\mu(x,y,x',y')}{\nux(x)} = \frac {\sum_{x'} \mu_\X(x,x')} {\nux(x)} = \frac 
{\sum_{x'} 
\nux(x,x')} {\nux(x)} = 1,
\]
which concludes the proof.

\section{Further details about the algorithm}\label{app:alg_details}
In this section we describe some further details about the derivation of our algorithm 
(\SOMCOT{}) that were omitted from the main text. Algorithm~\ref{alg:main2} provides a full pseudocode for \SOMCOT{}. 

At a high level, the algorithm aims to find the saddle point of the Lagrangian~\eqref{eq:Lagrangian_LP} by performing 
primal-dual updates for the two sets of variables $\pa{\mu,\lambda}$ and $\pa{\alpha,V}$, referred to as 
\emph{minimizing} and \emph{maximizing players}, respectively (or often simply call them min and max players). Both 
sets of players maintain a sequence of iterates $\pa{\mu_k,\lambda_k}$ and $\pa{\alpha_k,V_k}$, which are updated using 
versions of online stochastic mirror descent, described below in detail. For the updates, the $\mu$ and $\lambda$ 
players move in the direction of the negative gradient of the Lagrangian evaluated at 
$(\mu_k,\lambda_k;\alpha_k,V_k)$, and the $\alpha$ and $V$ players move in the direction of the positive gradient.

Since some of these gradients involve the occupancy measures $\nux$ and $\nuy$, they cannot be computed exactly without 
perfect knowledge of these distributions. However, since the dependence on $\nux$ and $\nuy$ is always linear, it is 
straightforward to obtain unbiased gradient estimators given only sample access to the chains $\MX$ and $\MY$. We 
provide a detailed guide for sampling from these distributions in Appendix~\ref{app:sampling}.

The remainder of the section provides a detailed derivation of the gradients and update rules
used in each iteration. 
We begin by introducing the Mirror Descent algorithm that
forms the basis of the update rules for each variable.

\begin{algorithm}
	\caption{Stochastic Optimization for Markov Chain Optimal Transport (\SOMCOT)}\label{alg:main2}
\begin{algorithmic}[1]
{
\REQUIRE{Convex sets $ \mathcal{D}_{\alphax}\subset \real^{\X\times\X\times\Y},
	\mathcal{D}_{\alphay}\subset\real^{\X\times\Y\times\Y}, \mathcal{D}_{V}\subset \real^{\X\times \Y}
	$,\\
Initial values
$ \mu_1, \lambdax_1, \lambday_1, \alphax_1, \alphay_1, V_1$, \\
Learning rates $ \eta, \etax, \etay, \betax, \betay, \beta > 0$}.
	\FOR{$ k=1, \ldots K-1 :$}
	\STATE{\textbf{Step 1: Draw samples from the Markov chains}}
	\STATE{Receive $(X_k,X'_k)\sim \nux $, $ (Y_k,Y'_k) \sim \nuy $}
	\STATE{\textbf{Step 2: Compute gradients or stochastic gradients}}
	\STATE{$g_{k, \mu}(x,y,x',y') \gets c(x,y) - \alphax_k(x, x',y) - \alphay_k(x,y, y') + \gamma V_k(x',y') - V_k(x,y) $}
	\STATE{$\wt{g}_{k,\lambdax}(y|x) \gets \mathbf{1}_{\{ X_k=x \}} \alphax_k(x, X'_k, y)$}
	\STATE{$ \wt{g}_{k, \lambday}(x|y) \gets \mathbf{1}_{\{ Y_k=y \}} \alphay_k(x, y,  Y'_k) $}
	\STATE{$ \wt{g}_{k, \alphax}(x, x', y) \gets \sum_{y'} \mu_k(x,y,x',y') - \mathbf{1}_{\{ X_k, X'_k = x, x' \}}
	\lambdax_k(y|x)$}
	\STATE{$ \wt{g}_{k, \alphay}(x, y, y') \gets \sum_{x'} \mu_k(x,y, x',y') - \mathbf{1}_{\{  Y_k, Y'_k= y, y'\}}
	\lambday_k(x|y)$}
	\STATE{$ g_{k, V}(x,y) \gets \sum_{x', y'} \mu_k(x,y, x',y') - (1- \gamma) \nu_0(xy) - \gamma \sum_{\hx,
	\hy} \mu_k(\hx,\hy, x,y) $}
	\STATE{\textbf{Step 3: Update primal variables}}
	\STATE{$ \mu_{k+1}(x,y, x',y') \propto \mu_k(x,y, x',y') \exp(- \eta g_{k, \mu}(x,y, x',y')) $}
	\label{line:mu_update}
	\STATE{$ \lambdax_{k+1}(y|x) \propto \lambdax_{k}(y|x) \exp( - \etax \wt{g}_{k,
	\lambdax}(y|x)) $} \label{line:lambdax_update}
	\STATE{$ \lambday_{k+1}(x|y) \propto \lambday_{k}(x|y) \exp ( - \etay \wt{g}_{k,
	\lambday}(x|y)) $}\label{line:lambday_update}
	\STATE{\textbf{Step 4: Update dual variables}}
	\STATE{$ \alphax_{k+1} \gets \Pi_{\mathcal{D}_{\alphax}}(\alphax_k - \betax \wt{g}_{k, \alphax})$} 
\label{line:alphax_update}
	\STATE{$ \alphay_{k+1} \gets \Pi_{\mathcal{D}_{\alphay}}(\alphay_k - \betay \wt{g}_{k, \alphay}) 
$}\label{line:alphay_update}
	\STATE{$V_{k+1} \gets \Pi_{\mathcal{D}_{V}}(V_k - \beta g_{k,V}) $}\label{line:V_update}
	\ENDFOR
	\STATE{Output $ \overline{\mu}_K = \frac{1}{K} \sum_{k=1}^{K} \mu_k $}.
}
\end{algorithmic}
\end{algorithm}

\subsection{Online Stochastic Mirror Descent}
\label{app:MD}
Online Stochastic Mirror Descent (OSMD) is an algorithm for the problem of online linear optimization, 
where in a sequence of rounds $k=1,2,\dots,K$, the following steps are repeated:
\begin{enumerate}
 \item The online learner picks a decision $z_k$ taking values in the vector space $\Zw$,
 \item the environment picks a linear function $g_k: \Zw \ra \real$,
 \item the online learner incurs loss $\iprod{g_k}{z_k}$, 
 \item the online learner observes an unbiased estimate $\wt{g}_k\in\Zw^*$ of the loss function.
\end{enumerate}
The sequence of steps above defines a filtration $\pa{\F_k}_k$, and the loss estimate $\wt{g}_k$ is assumed to satisfy 
$\EEcc{\wt{g}_k}{\F_{k-1}} =g_k$. Typically, the vectors $g_k$ are subgradients of a sequence of convex loss functions, 
and thus we will often refer to them with this term, and also call the vectors $\wt{g}_k$ stochastic subgradients (or 
simply stochastic gradients).
OSMD computes a sequence of updates based on these noisy gradient estimates and 
a convex and differentiable \emph{distance-generating function} $\Psi:\mathcal{Z}\to\real$. Concretely, OSMD operates 
with the Bregman divergence $\mathcal{B}_\Psi$ of $\Psi$, defined for each pair $z,z'\in\mathcal{Z}$ as
\begin{equation*}
	\mathcal{B}_\Psi(z\| z') = \Psi(z) - \Psi(z') - \siprod{\nabla \Psi(z')}{z-z'}.
\end{equation*}
OSMD starts with an initial point $z_1\in\mathcal{Z}$, and computes each subsequent iterate using the recursive update 
rule
\begin{equation}\label{eq:MD}
z_{k+1} = \argmin_{z \in \mathcal{Z}} \, \iprod{\wt{g}_k}{z} + \frac{1}{\eta}\mathcal{B}_\Psi(z \| z_k),
\end{equation}
where $\eta>0$ is called the learning rate.

Each of the update rules used by \SOMCOT follows from instantiating OSMD with a specific decision 
space $\Zw$, a distance-generating function $\Psi$ and a noisy subgradient estimator.
Concretely, we will make use of the following instances and corresponding update rules of MD, whose derivations are 
available in standard textbooks (e.g.~\citealt{Ora19}).

\begin{proposition}\label{prop:pgd}
When $\Zw=\real^d$ and $\Psi$ is the squared Euclidean norm defined as $\Psi(z)=\frac 1 2
\norm{z}_2^2$ for each $z\in\mathcal{Z}$, the OSMD update reduces to the \emph{projected stochastic 
gradient descent} update rule
\[
z_{k+1} = \Pi_{\mathcal{D}}(z_k - \eta \wt{g}_k),
\]
where $ \Pi_{\mathcal{\Zw}} $ is the orthogonal projection onto the set
$ \mathcal{\Zw} $ defined as $\Pi_{\mathcal{\Zw}}(x) = \argmin_{y \in \mathcal{\Zw}}\norm{x-y}_2$.
\end{proposition}

\begin{proposition}\label{prop:entreg}
When $\Zw=\Delta_\X$ is the probability simplex on a finite set $\X$ and $\Psi$ is the negative entropy
defined as $\Psi(p) = \sum_x p(x) \log p(x)$, $p\in\Delta_\X$, the OSMD update reduces to
\[
p_{k+1}(x) = \frac {p_k(x)e^{-\eta \wt{g}_k(x)}} {\sum_y p_k(y)e^{-\eta \wt{g}_k(y)}} \quad (\forall x\in\X).
\]
\end{proposition}

\begin{proposition}\label{prop:condentreg}
When $\Zw=\Delta_{\Y|\X}$ is the conditional simplex on finite sets $\X$ and $\Y$ and $\Psi$ is
the total negative entropy $\Psi(p) = \sum_{x,y} p(y|x) \log p(y|x)$, $p\in\Delta_{\Y|\X}$, the OSMD update reduces to
\[
p_{k+1}(y|x) = \frac {p_k(y|x)e^{-\eta \wt{g}_k(y|x)}} {\sum_{\bar y} p_k(\bar y|x)e^{-\eta \wt{g}_k(\bar y|x)}} \quad 
(\forall x,y\in\X\Y).
\]
\end{proposition}

\subsection{Primal updates}
In this section we derive the gradients and update rules of the primal variables $\mu$ and $\lambdax$. The gradient and 
update rule of $\lambday$ follow by symmetry. 

For $\mu$, first notice that any valid occupancy coupling $\mu$ is an element of the simplex $\Delta_{\X\Y\X\Y}$: 
summing the flow constraint in~\eqref{eq:bellman_flow_2} over $x$ and $y$ immediately yields 
$\sum_{x,y,x',y'}\mu(x,y,x',y')=1$. Thus, it is natural to enforce this constraint throughout the execution of the 
algorithm and use OSMD with the entropy regularizer given in Proposition~\ref{prop:entreg}. In order to derive the 
update rule, it remains to compute the gradients of the Lagrangian with respect to $\mu$, which is given as
\begin{equation}
\label{eq:grad_mu}
	\frac{\partial \LL}{\partial \mu}[\mu, \lambda ; \alpha, V](x,y, x',y') = c(x,y) - \alphax( x, 
x',y)- \alphay(x, y, y') + \gamma V(x', y') - V(x,y) .
\end{equation}
In each iteration $k$, the algorithm computes the gradient $ g_{k, \mu} = \frac{\partial \LL}{\partial \mu}
[\mu_k, \lambdax_k, \lambday_k ; \alphax_k, \alphay_k, V_k]  $, which can be used as the unbiased estimator $\wt{g}_k$. 
 Altogether, this yields the update rule on line~\ref{line:mu_update} of Algorithm~\ref{alg:main2}.

As for the $\lambdax$ variables, notice that Proposition~\ref{prop:causality_constraints} implies that $\lambdax$ 
belongs to the conditional simplex $ \Delta_{\Y|\X} = \{ \lambda\in\real_+^{\Y\times\X}, \forall x \in \X, \sum_{y} 
\lambda(y|x) = 1 \} $. Thus, it is natural to use the total negative entropy as regularization function (as suggested 
in Proposition~\ref{prop:condentreg}). For selecting the update direction, we note that the gradient of the Lagrangian 
with respect to $\lambdax$ is
\begin{equation}
\label{eq:grad_lambdax_2}
	\frac{\partial \LL}{\partial \lambdax}[\mu, \lambda ; \alpha, V](y|x) = \sum_{x'} \nux(x, x')
	\alphax(x, x', y) = \EEs{\mathbf{1}_{\{ X=x \}} \alphax(x, X', y)}{X, X' \sim \nux}.
\end{equation}
% Since we do not assume access to the transition kernel $\PX$, we cannot compute $\nux$ explicitly.
Thus, we can obtain a stochastic gradient estimate $\wt{g}_{k,\lambdax}$ of
$\frac{\partial \LL}{\partial \lambdax}[\mu_k, \lambdax_k, \lambday_k ; \alphax_k, \alphay_k, V_k]$ by
sampling a transition $(X_k,X_k')$ from $\nux$ and setting $\wt{g}_{k,\lambdax} = \mathbf{1}_{\{ X_k=x \}}
\alphax_k(x, X'_k, y) $. Putting things together, this yields the update rules on lines 
\ref{line:lambdax_update}--\ref{line:lambday_update} of Algorithm~\ref{alg:main2}.

\subsection{Dual updates}
We now move our attention to the dual variables. Again, we will derive the gradients and update
rules for $\alpha_\X$ and $V$, and the gradient and update rule for $\alpha_\Y$ follow by symmetry.
Since we are maximizing over the dual variables which are not restricted to any simplex, we update
them using projected (stochastic) gradient ascent, which is why the gradients are negated below. The feasible 
sets for the dual variables are chosen to enable using Lemma~\ref{lemma:constraint_rounding} for bounding the 
estimation error---see Appendix~\ref{app:main_error} for details.

The gradient of the Lagrangian with respect to $\alphax$ is defined as
\begin{align*}
	\frac{\partial \LL}{\partial \alphax}[\mu, \lambda ; \alpha, V](x, x',y) &= -\left(\sum_{y'} 
\mu(x,y,
	x',y') - \nux(x, x') \lambdax(y|x)\right)\\
												  &= -\EEs{\sum_{y'} \mu(x,y, x',y') - \mathbf{1}_{\{ X_k, X_k' = x, x' \}}
	\lambdax(y|x)}{X_k, X'_k \sim \nux}.
\end{align*}
Our algorithm will use the update direction $ \wt{g}_{k, \alphax}(x, x', y)= \sum_{y'} \mu_k(x,y, x',y') -
\mathbf{1}_{\{ X_k, X'_k = x, x' \}} \lambdax_k(y|x)  $. We will apply OSMD to update $ \alphax $
using a learning rate $\betax$ and the regularizer in Proposition~\ref{prop:pgd}, which yields
the update rule on lines \ref{line:alphax_update}--\ref{line:alphay_update} of Algorithm~\ref{alg:main2}.

The gradient of the Lagrangian with respect to $V$ is given by
\begin{equation*}
	\frac{\partial \LL}{\partial V}[\mu, \lambda ; \alpha, V](xy) = - \left( \sum_{x',y'} \mu(x,y,
	x',y') - (1-\gamma) \nu_0(x,y) - \gamma \sum_{\hx,\hy} \mu(\hx,\hy, x,y ) \right).
\end{equation*}
We use the update direction $ g_{k,V} = \sum_{x',y'} \mu_k(x,y, x',y') - (1 - \gamma) \nu_0(x,y) - \gamma \sum_{\hx,\hy}
\mu_k(\hx,\hy, x,y) $ and apply OSMD to update $ V $ using a learning rate $\beta$ and the
regularizer in Proposition~\ref{prop:pgd}, which yields the update rule on line~\ref{line:V_update} of 
Algorithm~\ref{alg:main2}.

\subsection{Sampling from $\nux$ and $\nuy$}\label{app:sampling}
A key step in constructing our gradient estimators (and thus running our algorithm) is drawing samples from the 
occupancy measures $\nux$ and $\nuy$. Here we provide further details about how to perform this operation in practice.

In order to generate a sample from the occupancy measure, we let $G$ be a geometric random variable with mean 
$\frac{1}{1-\gamma}$, and recall the definition of the marginal occupancy measure $\nux$ to write
\begin{align*}
	\nux(x, x') 
	&= (1- \gamma)\sum_{t=0}^{\infty}  \gamma^t \PP{X_t =x, X_{t+1}=x'}
	= \sum_{t=0}^{\infty} \PP{G = t} \PP{X_t =x, X_{t+1}=x'}
	\\
	&= \PP{X_{G} =x, X_{G+1}=x'}.
\end{align*}
Thus, one can obtain independent samples from $ \nux $ by first sampling a geometric stopping time $G$, sample a 
sequence $( X_0, X_1, \ldots, X_G, X_{G+1} )$, and keep the last pair of states $ X_{G}, X_{G+1}$. 

We remark that the task of sampling from an occupancy measure is common in reinforcement learning, and in particular it 
is necessary for correctly implementing policy gradient methods. To avoid sampling an entire sequence in each
iteration, it is standard practice to replace samples from the occupancy measure with arbitrary sample trajectories
generated by the Markov chain. Specifically, it is common to ignore discounting and draw samples
directly from trajectories in which consecutive state pairs are no longer independent.
We expect that, like most other RL algorithms, our method is also resilient to such abuse, and can be fed with sample 
pairs drawn from longer trajectories without resets or throwing away samples to ensure independence.

\allowdisplaybreaks
\section{Analysis}\label{app:analysis}
This section provides the complete details for the proof of our main result, Theorem~\ref{thm:PAC_bound_main}. 
Throughout the analysis, we will assume $\infnorm{c} \le 1$.
Completing the outline provided in Section~\ref{sec:analysis} requires filling two gaps: proving 
Lemma~\ref{lem:main_error}, and bounding the duality gap in terms of the regrets of the two players. These are 
respectively done in Sections~\ref{app:main_error} and~\ref{app:regret_bounds} below (with Section~\ref{app:rounding} 
providing additional technical tools for the proof of Lemma~\ref{lem:main_error}). Putting the two parts together 
complete the proof.

\subsection{Proof of Lemma~\ref{lem:main_error}}\label{app:main_error}
The majority of our theoretical analysis is dedicated to proving the error 
bound stated as Lemma~\ref{lem:main_error}, recalled here for convenience as
\begin{equation}
 \abs{\siprod{ \bmu_K- \mu^{*}}{c}} \le \GG_K(\mu^*,\lambda^*,\alpha^*,V^*).
\end{equation}
As a first step towards this proof,  we first need to define a technical tool that will allow 
us to quantify the constraint violations associated with the output $\bmu_K$. Indeed, one challenge in the analysis is 
that $\bmu_K$ does not necessarily satisfy the constraints~\eqref{eq:bellman_flow_2}--\eqref{eq:Y_causal} exactly. We 
quantify this effect by defining \emph{total absolute constraint violations} associated with the primal 
variables $\mu$, $\lambdax$ and $\lambday$ respectively by
\begin{align*}
\partial \F(\mu) &= \sum_{x,y}\abs{\sum_{x',y'} \mu(x,y, x'y') - \gamma \sum_{\hx,\hy} \mu(\hx,\hy, 
x,y) - (1-\gamma) \nu_0(x,y)}\\
 \partial \CX(\mu,\lambdax) &= \sum_{x,x',y}\abs{
 \sum_{y'} \mu(x,y,x',y') - \nux(x,x')\lambdax(y|x)}\\
 \partial \CY(\mu,\lambday) &= \sum_{x,y,y'}\abs{\sum_{x'} \mu(x,y,x',y') - \nuy(y,y') \lambday(x|y)}.
\end{align*}
For the sake of analysis, we will make use of a \emph{rounding procedure} that will convert $\bmu_K$ 
into a valid occupancy coupling $r(\bmu_K)$ that satisfies all constraints. Importantly, this rounding 
procedure never has to be executed in reality: it is only used as a device within the analysis.
The details of this rounding process (which is an adaptation of a method developed by \citealt{Calo_J_N_S_S24}) are 
provided in Appendix~\ref{app:rounding}. The following lemma provides an upper bound on the rounding error in terms of 
the total absolute constraint violations.
\begin{lemma}
\label{lemma:constraint_rounding}
Let $ \mu \in \Delta_{\X\Y\X\Y}$ and $ r(\mu) $ be its rounding (as defined in Appendix~\ref{app:rounding}), and 
$\lambdax$ and $\lambday$ be arbitrary. Then, we have 
\begin{equation}
\label{eq:constrain_rounding}
\norm{\mu - r(\mu)}_1  \leq \frac{3\CCX(\mu, \lambdax) + 3\CCY(\mu, \lambday) +
\bfc(\mu)}{1-\gamma}.
\end{equation}
\end{lemma}
The proof is provided along with all relevant definitions in Appendix~\ref{app:rounding}.
With this rounding process and its guarantees at hand, we can rewrite the absolute error between the cost estimate 
$\iprod{\bmu_K}{c}$ and the true cost $d(\MX,\MY) = \iprod{\mu^*}{c}$ as follows:
\begin{equation}\label{eq:error_decomp}
 \begin{split}
	\abs{\siprod{ \bmu_K- \mu^{*}}{c}} &\leq 
\siprod{r(\bmu_K) - \mu^{*}}{c} + \onenorm{r(\bmu_K) - 
\bmu_K}\infnorm{c}\\
						 &\leq \siprod{\bmu_K - \mu^*}{c} + 2 \norm{\bmu_K - r(\bmu_K)}_1\infnorm{c}.
 \end{split}
\end{equation}
Here, the first step follows from the triangle inequality and the crucially important fact that $\siprod{r(\bmu_K) - 
\mu^{*}}{c} \ge 0$ thanks to the feasibility of $r(\bmu_K)$ and the optimality of $\mu^*$.

It now only remains to relate the quantity appearing on the right-hand side of the above bound with the duality gap. To 
this end, we define the shorthand $\blambdax = \frac{1}{K} \sum_{k=1}^K \lambdax_k$ and $\blambday = \frac{1}{K} 
\sum_{k=1}^K \lambday_k$ and recall the choice
\[
 (\alpha^*,V^*) = \argmax_{\alpha \in \mathcal{D}_\alpha, V\in \mathcal{D}_V} \frac{1}{K}\sum_{k=1}^K 
\LL(\mu_k,\lambda_k;\alpha,V).
\]
Then, by plugging these variables into the Lagrangian, it is easy to check that
\[
 \frac 1K \sum_{k=1}^K\LL(\mu_k,\lambda_k;\alpha^*,V^*) = \iprod{\bmu_K}{c} + \frac{6\CCX(\bmu_K,\blambdax) + 
6\CCY(\bmu_K,\blambday) + 2 \bfc(\bmu_K)}{1-\gamma},
\]
which, by using Lemma~\ref{lemma:constraint_rounding}, implies the following bound:
\[
 \siprod{\bmu_K}{c} + 2\norm{ \bmu_K - r(\bmu_K)}_1 \le \frac 1K \sum_{k=1}^K\LL(\mu_k,\lambda_k;\alpha^*,V^*)
\]
On the other hand, it is easily verified that $\siprod{\mu^{*}}{c} = \LL(\mu^{*}, \lambda^{*} ; \alpha, V)$ holds for 
any choice of $\alpha$ and $V$, thanks to the fact that $\mu^*$ and $\lambda^*$ verify all the constraints of the LP. 
Putting this together with Equation~\eqref{eq:error_decomp}, we obtain that the error can be bounded in terms of the 
duality gap at the above-defined comparator $(\mu^*,\lambda^*,\alpha^*,V^*)$ as 
\[
 \abs{\siprod{ \bmu_K- \mu^{*}}{c}} \le \GG_K(\mu^*,\lambda^*,\alpha^*,V^*).
\]
This concludes the proof of Lemma~\ref{lem:main_error}.

\subsection{Rounded coupling and rounding error}
\label{app:rounding}
We describe the process and guarantees of rounding an (approximate) occupancy coupling $ 
\mu\in\real_+^{\X\Y\X\Y} $. We note that computing this rounding requires knowledge of $ \nu_X$, but 
this does not cause any practical problems since the rounding is only ever executed in the analysis.
For the rounding process itself, we first introduce the \emph{state-occupancy measure} $ \nu_\mu(x,y)= \sum_{x'
y'} \mu(x,y,x',y') $, and we define the associated \emph{transition coupling} $\pi_\mu$ as the kernel $\pi_\mu: 
\X\Y \ra \Delta_{\X\Y}$ with entries
\begin{align*}
	\pi_\mu(x^{\prime},y^{\prime}|x,y) =
\begin{cases}
	\frac{\mu(x,y,x^{\prime},y^{\prime})}{\nu_{\mu}(x,y)} \text{ if } \nu_\mu(x,y) \neq 0, \\
	\PX(x^{\prime}|x)\PY(y^{\prime}|y) \text{ otherwise} .
\end{cases}
\end{align*}
As shown by \citet{Calo_J_N_S_S24}, each transition coupling $\pi: 
\X\Y \ra \Delta_{\X\Y}$ induces a unique occupancy coupling $\mu^\pi$, and that the occupancy induced by $\pi_\mu$ is 
valid if and only if it equals $\mu$ (i.e., if $\mu^{\pi_\mu} = \mu$ holds). For more details, we refer to Appendix~B.2 
in \citet{Calo_J_N_S_S24}.

Following \citet{Calo_J_N_S_S24}, we will apply the rounding procedure of \citet[Algorithm~2]{ANWR17}  to $
\pi_{\mu} $ to obtain a valid transition coupling $r(\pi_\mu)$, and then extract the occupancy coupling 
induced by $r(\pi_\mu)$. 
More precisely, for two probability distributions $ p\in \Delta(\X) $, $ q\in \Delta(\Y) $, the set of valid couplings
is defined as $ \mathcal{U}_{p,q} = \{ P\in \real_+^{\X\times\Y} : P \cdot \mathbf{1}=p ; P^T \cdot \mathbf{1} =q \} $. For a
nonnegative matrix $ F \in \real_+^{\X\times\Y} $, the rounding procedure outputs a valid coupling $ r(F,p,q) \in
\mathcal{U}_{p,q}$. By Lemma 7 of \citet{ANWR17}, the rounded coupling satisfies
\begin{equation*}
	\norm{r(F,p,q) - F}_1 \leq  2(\norm{F\cdot \mathbf{1}} + \norm{F^T \cdot  \mathbf{1}}).
\end{equation*}
For completeness the procedure is detailed in Algorithm~\ref{alg:round}.

\begin{algorithm}
\caption{Rounding procedure for couplings\hspace*{-.2cm}}\label{alg:round}
\textbf{Input: } approximate coupling $F$, marginals $p$, $q$\\
$X \gets \diag(\min(p/(F\cdot\mathbf{1}),\mathbf{1}))$\\
$F' \gets XF$\\
$Y \gets \diag(\min(q/(F'^\top\cdot\mathbf{1}),\mathbf{1}))$\\
$F'' \gets F'Y$\\
$\text{err}_p=p-F''\cdot \mathbf{1}$, $\text{err}_q=q-F''^\top\cdot\mathbf{1}$\\
\textbf{Output: } $G\gets F'' + \text{err}_p\text{err}_q^\top / \onenorm{\text{err}_p}$
\end{algorithm}

This procedure is not symmetric, and thus we consider the following symmetrized procedure  defined as
\begin{equation*}
	\rsym(F,p,q) = \frac{r(F,p,q) + r(F^T,q,p)^T}{2}.
\end{equation*}

To obtain the rounded transition coupling, we apply the rounding procedure individually for each pair of states $x,y$. 
In particular, for a transition kernel $\pi_\mu:
\X\Y \ra \Delta_{\X\Y}$, we define its rounded counterpart $\wt{\pi} = r(\pi)$ with entries
\begin{equation*}
\wt{\pi}(\cdot |x,y) = \rsym(\pi(\cdot |x,y), \PX(\cdot |x),\PY(\cdot |y)).
\end{equation*}
With some abuse of notation, we will now denote as $ r(\mu) $ the occupancy coupling induced by $ r(\pi_{\mu}) $.
Because $  r(\pi_{\mu}) $ is a valid transition coupling, $ r(\mu) $ is a valid occupancy coupling.
The following derivations will relate the distance between $ r(\mu) $ and $ \mu $ to the total absolute constraint 
violations of $\mu$, thus providing a proof for Lemma~\ref{lemma:constraint_rounding}.

To make the subsequent derivations easier, we will define some handy notation. We first define the operator 
$E:\Delta_{\X\Y\X\Y}\ra\Delta_{\X\Y}$ via its action on any $\mu$ as $(E\mu)(x,y) = 
\sum_{\hx,\hy} \mu(\hx,\hy,x,y)$, and note that this allows us to rewrite the flow condition~\eqref{eq:bellman_flow_2} 
in the form $\nu_\mu = \gamma E\mu + (1-\gamma) \nu_0$. For a 
state-distribution $\nu\in\Delta_{\X\Y}$ and a kernel $\pi:\X\Y \ra \Delta_{\X\Y}$, we define the 
composition $\nu \circ \pi$ as the distribution $p$ with entries $p(x,y,x',y') = \nu(x,y)\pi(x',y'|x,y)$. We will 
specifically use the notation $\Delta(\mu) = \nu_{\mu}\circ(r(\pi_{\mu})-\pi_{\mu})$. Armed with all this notation, we 
bound the $\ell_1$ distance between $r(\mu)$ and $\mu$ as
\begin{align*}
	\norm{r(\mu)-\mu}_1 &= \norm{\nu_{r(\mu)}\circ r(\pi_{\mu}) - \nu_{\mu}\circ \pi_{\mu}}\\
					   &= \norm{\nu_{r(\mu)}\circ r(\pi_{\mu}) - \nu_{\mu}\circ r(\pi_{\mu}) +
				   \nu_{\mu}\circ r(\pi_{\mu}) - \nu_{\mu}\circ \pi_{\mu}}\\
					   &\leq  \norm{\nu_{r(\mu)} - \nu_{\mu}}_1 +
					   \norm{\nu_{\mu}\circ(r(\pi_{\mu})-\pi_{\mu})}_1 \\
					   &= \norm{ \nu_{r(\mu)} - (1- \gamma) \nu_0 + (1 - \gamma) \nu_0 -
					   \nu_{\mu}}_1 + \norm{\Delta(\mu)}_1 \\
					   &= \norm{\gamma E r(\mu) - \gamma E \mu + \left[ \gamma E \mu + (1-\gamma)
					   \nu_0 - \nu_{\mu} \right]}_1 + \norm{\Delta (\mu)}_1 \\
					   & \leq \gamma\norm{r(\mu) - \mu}_1 + \bfc(\mu) + \norm{\Delta (\mu)}_1,
% 					   + 3					   \norm{\myc(\mu, \lambday)}_1 + 3 \norm{\mxc(\mu,\lambdax)}_1 .
\end{align*}
where the second-to-last line uses the fact that $ r(\mu) $ is a valid occupancy coupling and as such satisfy the
flow condition~\eqref{eq:bellman_flow_2}, and we have recalled  the definition of $\bfc(\mu)$ stated in the main text. 
After reordering, we obtain
\[
 \norm{r(\mu)-\mu}_1 \le \frac{\bfc(\mu) + \norm{\Delta (\mu)}_1}{1-\gamma},
\]
and thus it remains to upper bound $\onenorm{\Delta(\mu)}$. This is done in the following lemma, using which concludes 
the proof of Lemma~\ref{lemma:constraint_rounding}.

\begin{lemma}
\label{lemma:Delta_bound}
For any $\mu\in\Delta_{\X\Y\X\Y}$, and any $\lambdax:\X\ra\Delta_{\Y}$ and $\lambday:\Y\ra\Delta_{\X}$, we have 
$\norm{\Delta(\mu)}_1 \leq 3 \myc(\mu, \lambday) + 3 \mxc(\mu, \lambdax)$.
\end{lemma}

\begin{proof}
By Lemma~7 of \citet{ANWR17}, we have that for arbitrary state pairs $x,y$, the following is satisfied:
\begin{align*}
	&\norm{r(\pi_{\mu})(\cdot |x,y) - \pi_{\mu}(\cdot |x,y)}_1 \\
	&\qquad\qquad\leq  2\left[
	\sum_{x^{\prime}} \left\lvert\PX(x^{\prime}|x) - \sum_{y^{\prime}} \pi_\mu(x^{\prime},y^{\prime}|x,y)\right\rvert + 
\sum_{y^{\prime}}
\left\lvert\PY(y^{\prime}|y) - \sum_{x^{\prime}} \pi_\mu(x^{\prime},y^{\prime}|x,y)\right\rvert \right]	.
\end{align*}
By symmetry, this directly gives
\begin{align*}
	&\norm{\rsym(\pi_{\mu})(\cdot |x,y)- \pi_{\mu}(\cdot |x,y)}_1\\
	&\qquad\qquad\leq  \frac{3}{2}\left[
	\sum_{x^{\prime}} \left\lvert\PX(x^{\prime}|x) - \sum_{y^{\prime}} \pi_\mu(x^{\prime},y^{\prime}|x,y)\right\rvert
	+ \sum_{y^{\prime}}
\left\lvert\PY(y^{\prime}|y) - \sum_{x^{\prime}} \pi_\mu(x^{\prime},y^{\prime}|x,y)\right\rvert \right]	.
\end{align*}
Now, multiplying both sides by $ \nu_{\mu}(x,y) $, we get
\begin{align*}
	&\Delta(\mu)(x,y) =\nu_{\mu}(x,y) \norm{\rsym(\pi_{\mu})(\cdot |x,y) - \pi_{\mu}(\cdot |x,y)}_1\\
			&\leq  \frac{3}{2}\left[
	\sum_{x^{\prime}} \left\lvert\PX(x^{\prime}|x) \nu_{\mu}(x,y) - \sum_{y^{\prime}} \mu(x,y,x',y')\right\rvert +
	\sum_{y^{\prime}}
\left\lvert\PY(y^{\prime}|y) \nu_{\mu}(x,y) - \sum_{x^{\prime}} \mu(x,y,x',y')\right\rvert \right]\hspace*{-4pt}.
\end{align*}

The first term on the right-hand side of the above expression can be bounded as follows:
\begin{align*}
\sum_{x'}&\left\lvert\PX(x'|x) \nu_{\mu}(x,y) - \sum_{y'} \mu(x,y,x',y')\right\rvert \\
&= \sum_{x'}\left\lvert
    \PX(x'|x) \nu_{\mu}(x,y)
    - \nux(x, x') \lambdax(y|x)
    + \nux(x, x') \lambdax(y|x)
    - \sum_{y'} \mu(x,y,x',y')
\right\rvert \\
&\overset{(\textit{i})}{\leq} \sum_{x'}\left\lvert
    \PX(x'|x) \nu_{\mu}(x,y)
    - \nux(x, x') \lambdax(y|x)
\right\rvert
+ \left\lvert
    \nux(x, x') \lambdax(y|x)
    - \sum_{y'} \mu(x,y,x',y')
\right\rvert \\
&= \sum_{x'} \PX(x'|x) \left\lvert
    \nu_{\mu}(x,y) - \nux(x) \lambdax(y|x)
\right\rvert
+ \mxc(\mu, \lambdax) \\
&\overset{(\textit{ii})}{=} \left\lvert
    \nu_{\mu}(x,y) - \nux(x) \lambdax(y|x)
\right\rvert
+ \mxc(\mu, \lambdax)
\\
&\overset{(\textit{iii})}{=} \left\lvert
    \sum_{x',y'} \mu(x,y,x',y')
    - \sum_{x'} \nux(x)\PX(x'|x) \lambdax(y|x)
\right\rvert
+ \mxc(\mu, \lambdax) \\
&\leq \sum_{x'} \left\lvert
    \sum_{y'} \mu(x,y,x',y')
    - \nux(x, x') \lambdax(y|x)
\right\rvert
+     \mxc(\mu, \lambdax)\\
&= 2 \mxc(\mu, \lambdax).
\end{align*}
Here, we used the triangle inequality for (\textit{i}) and the fact that $ \sum_{x'}\PX(x'|x)=1 $ for (\textit{ii}) and 
(\textit{iii}). The proof is concluded by repeating the same argument for the constraint violations $\myc(\mu, 
\lambday)$, and plugging the results back into the previous inequalities.
\end{proof}

\subsection{Regret bounds of the primal and dual sequence}\label{app:regret_bounds}
This section provides an upper bound on the duality gap as defined in Equation~\eqref{eq:duality_gap}, in terms of the 
\emph{regrets} of the two set of algorithms controlling the primal and dual variables. Recalling the convention 
established in Appendix~\ref{app:alg_details}, we will refer to the algorithms as the min- and max-players, with their 
regrets respectively defined as
\begin{align*}
 \regretmax_K(\alpha^*,V^*) &= \sum_{k=1}^K \bpa{\LL(\mu_k,\lambda_k;\alpha^*,V^*)  - 
\LL(\mu_k,\lambda_k;\alpha_k,V_k)}
\\
 \regretmin_K(\mu^*,\lambda^*) &= \sum_{k=1}^K 
\bpa{\LL(\mu_k,\lambda_k;\alpha_k,V_k) - 
\LL(\mu^*,\lambda^*;\alpha_k,V_k)},
\end{align*}
where our notation emphasizes that each regret is measured against the comparators $\alpha^*,V^*$ and 
$\mu^*,\lambda^*$. 
With this notation, the duality gap can be rewritten as 
\begin{equation}\label{eq:duality_decomposition}
\GG_K(\mu^*,\lambda^*;\alpha^*,V^*) = \frac{\regretmax_K(\alpha^*,V^*) + \regretmin_K(\mu^*,\lambda^*)}{K}.
\end{equation}
With a mild abuse of our earlier notation, we write out the full expression of the Lagrangian in terms of the 
$\alphax,\alphay$ and $\lambdax,\lambday$ variables as 
$\LL(\mu,\lambdax,\lambday;\alphax,\alphay,V)$. The regret terms that need to be bounded can be further decomposed in 
terms of the following individual terms defined for each set of primal and dual variables:
\begin{align*}
 \regretmax_K(\alphax^*) &= \sum_{k=1}^K
 \LL(\mu_k, \lambdax_k, \lambday_k ; \alphax^{*}, \alphay^{*}, V^{*}) - \LL(\mu_k, \lambdax_k, \lambday_k ;
	\alphax_k, \alphay^{*}, V^{*})
	\\
	\regretmax_K(\alphay^*) &= \sum_{k=1}^K
 \LL(\mu_k, \lambdax_k, \lambday_k ; \alphax_k, \alphay^{*}, V^{*}) - \LL(\mu_k, \lambdax_k, \lambday_k ;
	\alphax_k, \alphay_k, V^{*})
	\\
	\regretmax_K(V^*) &= \sum_{k=1}^K\LL(\mu_k, \lambdax_k, \lambday_k; \alphax_k, \alphay_k, V^{*}) - \LL(\mu_k, 
\lambdax_k, \lambday_k
	; \alphax_k, \alphay_k, V_k)
	\\
	\regretmin_K(\mu^*) &= \sum_{k=1}^K \LL(\mu_k, \lambdax_k, \lambday_k; \alphax_k, \alphay_k, V_k) - 
\LL(\mu^{*}, \lambdax_k, \lambday_k; \alphax_k, \alphay_k, V_k)
\\
\regretmin_K(\lambdax^*) &= \sum_{k=1}^K \LL(\mu^*, \lambdax_k, \lambday_k; \alphax_k, \alphay_k, V_k) - \LL(\mu^*, 
\lambdax^{*}, \lambday_k 	; \alphax_k, \alphay_k, V_k) 
\\
\regretmin_K(\lambday^*) &= \sum_{k=1}^K\LL(\mu^*, \lambdax^*, \lambday_k; \alphax_k, \alphay_k, V_k) - \LL(\mu^*, 
\lambdax^*, \lambday^*
	; \alphax_k, \alphay_k, V_k) 
\end{align*}
Thanks to the bilinearity of the Lagrangian, each of these terms can be seen as the regret of an online learning 
algorithm with linear loss / gain functions and decision variables taking values in a convex decision space 
$\Zw$ (embedded within some Euclidean space). In particular, each of these regrets can be written in the 
following form for some sequences $\pa{g_k}_k\in\real^d$, $\pa{z_k}_k\in \Zw$ and $z\in \Zw$:
\[
 \regret_K(z^*) = \sum_{k=1}^{K} \siprod{g_k}{z_k - z^*}.
\]
As noted in Section~\ref{app:alg_details}, our algorithm can be understood as running an instance of Online Stochastic 
Mirror Descent (OSMD) for each set of variables, and thus each regret term can be bounded 
using standard results. One challenge for the analysis is that the comparator points $\alpha^*$ and $V^*$ are 
chosen in a data-dependent manner. This is not easily handled by standard tools in online learning, but can still be 
treated with some relatively more advanced tools that are common in the context of saddle-point optimization (most 
notably, using techniques of \citealt{NJLS09,RS17}). In particular, we will use the following general result to 
bound the regrets of each player in the analysis below.
\begin{lemma}
\label{lemma:OMD_bound}
Let $ z^* \in \Zw$ be a potentially data-dependent comparator and assume that $ \Psi $ is $ \lambda $-strongly convex 
with respect to some norm $  \norm{\cdot } $ whose dual is denoted by $\norm{\cdot}_*$. Furthermore, suppose that 
$\sup_{z,z'\in\Zw} \norm{z - z'} \le C$ holds for some constant $C>0$. Then, for any 
$\wt{\eta}>0$, the sequence $ \pa{z_k}_k $ produced by OSMD satisfies the following bound 
with probability at least
$1-\delta$:
\begin{align*}
\label{eq:OMD_bound}
	\sum_{k=1}^{K} \siprod{g_k}{z_k - z^*} \leq& 
	\frac{\BB_{\Psi}\pa{z \middle\| z_1}}{\eta} + \frac{\eta}{2 \lambda}
	\sum_{k=1}^{K}\norm{g_k}_*^2
	\\
	&+\frac{\BB_{\Psi}\pa{z \middle\| z_1}}{\wt{\eta}} + \frac{\wt{\eta}}{2\lambda} \sum_{k=1}^{K}\norm{g_k - 
\wt{g}_k}_*^2 + C\sqrt{2 \sum_{k=1}^{K}\norm{g_k - \wt{g}_k}_*^2 \log \frac {1}{\delta}}.
\end{align*}
\end{lemma}
While composed of standard elements, we provide the proof for the sake of completeness in 
Appendix~\ref{app:online_learning}. The regret bound itself can be simplified in two different ways, 
depending on whether or not the algorithm in question uses deterministic or stochastic gradients: for deterministic 
updates, we have $g_k = \wt{g}_k$ and we can choose $1/\wt{\eta} = 0$, whereas for stochastic updates the choice 
$\wt{\eta} = \eta$ is more natural. This is how we will apply the lemma to each regret term below.
In what follows, we 
will instantiate this bound to bound the regrets of all players listed above, which will require establishing \emph{i}) 
the strong-convexity properties of the regularization functions, \emph{ii)} bounds on the Bregman divergences between 
the initial points and the comparators and \emph{iii}) bounds on the dual norms of the gradients and the gradient 
noise. This is done case by case in the following subsections.

\subsubsection{Regret of the $ \protect\alpha$-players}
The policy of the $\alpha$-players is to run projected online stochastic gradient ascent on the feasible set $ \Zw = 
B_{\infty}(\frac{6}{1-\gamma})$, and unbiased gradient estimators with elements defined respectively as
\[
\wt{g}_{k,\alphax}(x, x', y)= \sum_{y'} \mu_k(x,y, x',y') - \mathbf{1}_{\{ X_k, X'_k = x, x' \}} \lambdax_k(y|x)
\]
and 
\[
\wt{g}_{k,\alphay}(x, y, y')= \sum_{x'} \mu_k(x,y, x',y') - \mathbf{1}_{\{ Y_k, Y'_k =y, y' \}} \lambday_k(x|y).
\]
The following lemma provides an upper bound on each of the 
two $\alpha$-players.
\begin{lemma}
\label{lemma:alpha_bound_POGD_PAC}
	With probability at least $ 1- \delta $, the regret of the $ \alphax $-player is bounded as\todoG{I got a better 
constant factor here by choosing $\alpha_1 = 0$ for both players.}
\begin{equation}
\label{eq:alpha_bound_POGD_PAC_x}
	\regretmax_K(\alphax^*) \leq \frac{18 \abs{\X}^2\abs{\Y}}{(1- \gamma)^2 \betax}  +
	4\betax K  + \sqrt{\frac{72K}{(1- \gamma)^2} \log\frac{2}{\delta}},
\end{equation}
and the regret of the $\alphay$-player is bounded as 
\begin{equation}
\label{eq:alpha_bound_POGD_PAC_y}
	\regretmax_K(\alphay^*) \leq \frac{18 \abs{\Y}^2\abs{\X}}{(1- \gamma)^2 \betay}  +
	4 \betay K  + \sqrt{\frac{72K}{(1- \gamma)^2} \log\frac{2}{\delta}}.
\end{equation}
\end{lemma}
\begin{proof}
We prove the claim for $\alphax$, and the result for $\alphay$ will follow by symmetry.
 We start by noting that the gradient estimators and the gradients satisfy $\norm{\widetilde{g}_{k,\alphax}}_1 \leq 2$ 
and 
$\norm{g_{k,\alphax} - \widetilde{g}_{k,\alphax}}_1 \leq 2$. Indeed, this can be verified easily as 
\begin{equation*}
	\norm{\widetilde{g}_{k,\alphax}}_1 \leq \sum_{x,y,x',y'} \mu_k(x,y,x',y') + \sum_{x, x',y} \mathbf{1}_{\{ X,X' = x, 
x' \}}
	\lambdax_k(y|x) = 2,
\end{equation*}
because of the normalization of both $ \mu_k$ and $\lambdax_k $.
Similarly, we have
\begin{equation*}
	\norm{g_{k,\alphax} - g_{k,\alphax}}_1 \leq \sum_{x, x', y} \pa{\mathbf{1}_{\{ X,X' = x, x' \}} + 
\nux(x,	x')} \lambdax_k(y|x) = 2.
\end{equation*}
Furthermore, $\norm{\alphax^{*} - \alphax_1}_{\infty} \leq \frac{6}{1-\gamma}$ trivially holds thanks to the 
definition of the domain of $\alphax$ and the choice $\alphax_1 = 0$. Finally, notice that $\Psi$ is 1-strongly convex 
with respect to $\twonorm{\cdot}$, and thus Lemma~\ref{lemma:OMD_bound}  (with the choice $\wt{\eta} = \eta = 
\betax$) immediately implies the claim after using the relations $\twonorm{\widetilde{g}_{k,\alphax}} \le 
\onenorm{\widetilde{g}_{k,\alphax}} \le 2$ 
and $\norm{\alphax^{*} - \alphax_1}_2^2 \le \norm{\alphax^{*} - \alphax_1}_{\infty}^2 \leq \frac{36 
\abs{\X}^2\abs{\Y}}{\pa{1-\gamma}^2}$.
\end{proof}

\subsubsection{Regret of the $ V $-player}
Similarly to the $\alpha$-players, the $V$-player employs online gradient ascent on the feasible set $ \Zw = 
B_{\infty}(\frac{2}{1-\gamma})$, with entries of the gradients given in each round as
\[
 g_{k,V}(x,y) = \sum_{x',y'} \mu_k(x,y, x',y') - (1 - \gamma) \nu_0(x,y) - \gamma \sum_{\hx,\hy}
\mu_k(\hx,\hy, x,y).
\]
The following lemma gives a bound on its 
regret.
\begin{lemma}
\label{lemma:beta_bound_POGD}
The regret of the $ V$-player is bounded as \todoG{I got an improvement here by choosing $V_1 = 0$.}
\begin{equation}
\label{eq:beta_bound_POGD}
	\regretmax_K(V^*) \leq \frac{4\abs{\X}\abs{\Y}}{\beta(1-\gamma)^2} + 2\beta K .
\end{equation}
\end{lemma}
\begin{proof}
Since the $V$-player employs deterministic gradients, we will apply Lemma~\ref{lemma:OMD_bound} with $1/\wt{\eta} = 
0$, and bound the Euclidean norms of the comparator $V^*$ and the gradients. By the choice of the feasible set for 
$V^*$ and the choice $V_1 = 0$, we immediately have $\twonorm{V^* - V_1} \le 
\abs{\X}\abs{\Y}\norm{V^{*}-V_1}_{\infty}^2 \le \frac{4 \abs{\X}\abs{\Y}}{(1- \gamma)^2}$. Furhermore, evaluating the 
gradient of the Lagrangian with respect to $V$, we get
\begin{equation*} 
	\norm{g_{k,V}}_1 \leq (1 - \gamma) \sum_{x,y} \nu_0(x,y) + (1+\gamma) \sum_{x,y ,x',y'}\mu(x,y,x',y') = 2,
\end{equation*}
which in turn implies $\norm{g_{k,V}}_2 \le \norm{g_{k,V}}_1 \le 2$. Plugging these results in the bound of 
Lemma~\ref{lemma:OMD_bound} concludes the proof.
\end{proof}

\subsubsection{Regret of the $ \mu $-player}
The $\mu$-player plays OSMD with entropy regularization, and gradients with elements defined as
\begin{equation*}
	g_{k, \mu}(x,y,x',y') = c(x,y) - \alphax_k(x, x',y) - \alphay_k(x,y,y') + \gamma V_k(x', y') - V_k(x,y).
\end{equation*}
The following bound gives a bound on the regret of this player.
\begin{lemma}
The regret of the $\mu$-player is bounded as \todoG{Got a worse constant here after recalculating the bound on the 
infinity norm}
 \[
  \regretmin_K(\mu^*) \le \frac{\log\bpa{\abs{\X}^2\abs{\Y}^2}}{\eta}+ \frac{200 \eta K }{(1-\gamma)^2}.
 \]
\end{lemma}
\begin{proof}
 The proof follows from noticing that the regularization function $\Psi$ is $1$-strongly convex with respect to the 
norm $\onenorm{\cdot}$, and that the dual norm of the gradients is bounded as $\norm{g_{k, \mu}}_{\infty} \leq 
\frac{20}{1-\gamma}$. Indeed, this follows by upper-bounding each entry of the gradient as
\begin{align*}
	\abs{g_{k,\mu}(x,y,x',y')} &\leq c(x,y) + \abs{\alphax_k(x, x',y')} + \abs{\alphay_k(x,y,y')} + \gamma\abs{V_k(x',y')} 
+ \abs{V_k(x,y)}
	\\
				  &\leq 1 + \frac{12}{1-\gamma} + \frac{4(\gamma+1)}{1-\gamma} = \frac{17 + 3\gamma}{1-\gamma} \le 
				  \frac{20}{1-\gamma}.
\end{align*}
Finally, we recall the choice of $\mu_1$ being uniform over $\X\Y\X\Y$, and the standard result that the relative 
entropy between any distribution and $\mu_1$ is equal to $\log(\abs{\X}^2\abs{\Y}^2)$.
\end{proof}

\subsubsection{Regret of the $ \protect\lambda$-players}
The regret analysis of the $\lambda$-players is slightly nonstandard. Focusing on the $\lambdax$-player here, we note 
that the updates correspond to using OSMD on the decision space 
$ \Zw = \{ \lambda \,:\, \forall x,y \; \lambda(y|x) \geq 0 \, ; \forall x
	\; \sum_{y} \lambda(y|x) = 1 \} $
with the following choice of regularization function:
\[
 \Psi(\lambda) = \sum_{x} \sum_{y} \lambda(y|x)
\log(\lambda(y|x)).
\]
As we show in Lemma~\ref{lemma:strong_convexity_entropy_cond}, this regularization function is $1$-strongly convex with 
respect to the \emph{2-1 group norm} defined for each $\lambda \in \Zw$ as 
\[ 
\norm{\lambda}_{2,1} = 
\sqrt{\sum_{x}\left( \sum_{y} |\lambda(y|x)| \right)^2}.
\]
It is easy to verify that the corresponding dual norm is the 
\emph{$2-\infty $ group norm} defined as $ \norm{g}_{2, \infty} = \sqrt{\sum_{x}(\max_{y}|g(x,y)|)^2} $.
We also recall that the updates make use of the following unbiased estimate of the gradient:
\begin{equation}
	\widetilde{g}_{k, \lambdax}(x,y) = \mathbf{1}_{\{ X=x \}} \alphax_k(X,X',y).
\end{equation}
With these facts at hand, we prove the following bound on the regret of the $\lambda$-players.
\begin{lemma} \label{lemma:lambdax_bound_OMD}
With probability at least $1-\delta$, the regret of the $\lambdax$ player and is 
bounded as 
\begin{equation}
	\regretmax_K(\lambdax^*) \leq \frac{\abs{\X}\log\abs{\Y}}{\etax}+	\frac{90\etax K}{(1-\gamma)^2} + 
\sqrt{\frac{288 \abs{\X}K \log \left(\frac{2}{\delta}\right)}{(1-\gamma)^2}}.
\end{equation}
and the regret of the $\lambday$-player is bounded as 
\begin{equation}
	\regretmax_K(\lambday^*) \leq \frac{\abs{\Y}\log\abs{\X}}{\etay}+	\frac{90\etay K}{(1-\gamma)^2} + 
\sqrt{\frac{288 \abs{\Y} K \log \left(\frac{2}{\delta}\right)}{(1-\gamma)^2}}.
\end{equation}
\end{lemma}
\begin{proof}
We provide a complete proof for $\lambdax$, and note that the result for $\lambday$ is analogous.
For this case, notice that Lemmas~\ref{lemma:OMD_bound} and~\ref{lemma:strong_convexity_entropy_cond} suggest that we 
should first obtain upper-bounds on the magnitude of the gradients in terms of their $2,\infty$-group norms, and thus 
we first establish that
\begin{align*}
	\norm{\widetilde{g}_{k, \lambdax}}_{2, \infty} &= \sqrt{\sum_{x}\left( \max_y |\mathbf{1}_{\{ X=x \}}
	\alphax(X,X',y)| \right)^2}
						       \leq \sqrt{\sum_{x}\mathbf{1}_{\{ X=x
						       \}}\norm{\alphax_k}_{\infty}^2} 
						       = \norm{\alphax_k}_{\infty}.
\end{align*}
Note that the latter is upper-bounded as $\norm{\alphax_k}_{\infty} \leq \frac{6}{1- \gamma}$ by construction.
Further observing that the true gradient norm can be bounded via the same argument as $\norm{\widetilde{g}_{k, 
\lambdax}}_{2, \infty} \le \frac{6}{1- \gamma}$, we also have
\begin{equation*}
	\norm{g_{k,\lambdax} - \widetilde{g}_{k,\lambdax}}_{2, \infty} \leq \norm{g_{k,\lambdax}}_{2, \infty} +
	\norm{\widetilde{g}_{k,\lambdax}}_{2, \infty} \leq \frac{12}{1- \gamma}.
\end{equation*}
Finally, since $ \lambdax_1 (\cdot|x)$ is chosen as the uniform distribution over $\Y$ for all $x$, we have $ 
\BB_{\Psi}(\lambdax^{*} \| \lambdax_1) \le \abs{\X} \log \abs{\Y}$, and the primal-norm distance satisfies 
$\norm{\lambda^* - \lambda}\le 2\sqrt{\abs{\X}}$. Now, the claim follows from using Lemma~\ref{lemma:OMD_bound} with 
$\wt{\eta} = \etax$.
\end{proof}

\subsection{Proof of Theorem~\ref{thm:PAC_bound_main}}
The proof of the theorem now follows from putting together Lemma~\ref{lem:main_error} with the regret decomposition in 
Equation~\eqref{eq:duality_decomposition}, and combining 
Lemmas~\ref{lemma:alpha_bound_POGD_PAC}--\ref{lemma:lambdax_bound_OMD}. Taking a union bound over the two 
probabilistic claims of Lemma~\ref{lemma:alpha_bound_POGD_PAC} and~\ref{lemma:alpha_bound_POGD_PAC}, this gives that 
the following bound holds with probability at least $1-2\delta$:
\begin{align*}
\abs{\siprod{\bmu_K - \mu^{*}}{c}}  \leq& \frac{18 |\X|^2|\Y|}{\betax K(1- \gamma)^2 }  +
	4 \betax + \sqrt{\frac{72}{K(1- \gamma)^2} \log\frac{2}{\delta}} \\ %alphax
	&+ \frac{18 |\X||\Y|^2}{\betay K(1- \gamma)^2 }  +
	4 \betay + \sqrt{\frac{72}{K(1- \gamma)^2} \log\frac{2}{\delta}} \\ %alphay
	&+ \frac{4|\X||\Y|}{\beta K(1-\gamma)^2} + 2\beta  \\ % V
	&+  \frac{2\log(|\X||\Y|)}{\eta K}+ \frac{200 \eta}{2(1-\gamma)^2} \\ % mu
	&+ \frac{|\X|\log (|\Y|)}{\etax K} + \frac{90\etax}{(1- \gamma)^2} 
	+ \sqrt{\frac{288 |\X| \log \frac{2}{\delta}}{K(1-\gamma)^2}} \\ % lambdax
	&+ \frac{|\Y|\log |\X|}{\etay K} + \frac{90 \etay}{(1- \gamma)^2} +
	\sqrt{\frac{288 |\Y| \log \frac{1}{\delta}}{K(1-\gamma)^2}}. % lambday
\end{align*}
Setting 
$\betax = \sqrt{\frac{9|\X|^2|\Y|}{2(1- \gamma)^2K}}$,
$\betay = \sqrt{\frac{9|\X||\Y|^2}{2(1- \gamma)^2K}}$,
$\beta = \sqrt{\frac{2|\X||\Y|}{(1- \gamma)^2K}}$,
$\etax = \sqrt{\frac{(1- \gamma)^2|\X|\log |\Y|}{90 K}}$,
$\etax = \sqrt{\frac{(1- \gamma)^2|\Y|\log |\X|}{90 K}}$,
$\eta  = \sqrt{\frac{(1- \gamma)^2\log(|\X||\Y|)}{100 K}}$, 
the bound becomes 
\begin{align*}
|\siprod{\bmu_K - \mu^{*}}{c}| \leq& \frac{12\sqrt{2|\X||\Y|}\bpa{\sqrt{|\X|}+\sqrt{|\Y|}}}{(1- \gamma)\sqrt{K}} 
	+ \frac{4\sqrt{2|\X||\Y|}}{(1-\gamma)\sqrt{K}} \\ % V
	&+ \frac{3\sqrt{10|\X|\log |\Y|}}{(1- \gamma)\sqrt{K}} + \frac{3\sqrt{10|\Y|\log |\X|}}{(1- \gamma)\sqrt{K}}
	+  \frac{40\sqrt{\log(|X|^2|Y|^2)}}{(1- \gamma)\sqrt{K}} \\
	&+ \frac{12 \sqrt{\log \frac 1\delta}}{(1- \gamma) \sqrt{K}} 
	+ \frac{12 \sqrt{2|\X| \log \frac{1}{\delta}}}{(1-\gamma)\sqrt{K}}
	+ \frac{12 \sqrt{2|\Y| \log \frac{1}{\delta}}}{(1-\gamma)\sqrt{K}}\\
	=& \OO\pa{\sqrt{\frac{(|\X||\Y|\pa{|\X|+|\Y|} + |\X|\log \frac{|\Y|}{\delta} + |\Y|\log \frac{|\X|}{\delta}} {(1- 
\gamma)^2K} }}.
\end{align*}
This concludes the proof.

\section{Online learning: The proof of Lemma~\ref{lemma:OMD_bound}}\label{app:online_learning}

This section is dedicated to proving the general regret bound we use throughout the analysis for upper-bounding the 
regret of each player, Lemma~\ref{lemma:OMD_bound}. As mentioned in Appendix~\ref{app:regret_bounds}, the main 
challenge that we need to deal with is that the comparators for some of the regret terms are data dependent, which 
requires some additional steps that are typically not necessary in regret analyses. For concreteness, we adapt the 
notation of Lemma~\ref{lemma:OMD_bound} and write the regret against comparator $z^*$ as
\[
 \regret_K(z^*) = \sum_{k=1}^K \iprod{g_k}{z_k - z^*} = 
 \underbrace{\sum_{k=1}^{K}\iprod{\widetilde{g}_k}{z_k -z^{*}}}_{R_K} + \underbrace{\sum_{k=1}^{K}\siprod{g_k - 
\widetilde{g}_k}{z_k - z^{*}}}_{M_K},
\]
where in the second equality we also added some terms corresponding to the stochastic gradient $\wt{g_k}$. 
Here, the first term $R_K$ corresponds to the regret of the online learning algorithm on the sequence of stochastic 
gradients $\wt{g}_k$, which can be upper-bounded using standard tools of online learning. For the second term, notice 
that the stochastic gradient satisfies $\EEcc{\wt{g}_k}{\F_{k-1}} = g_k$, and thus if $z^*$ is independent of the 
sequence of stochastic gradients, the second term $M_K$ in the above decomposition is a martingale. However, this is 
no longer true if $z^*$ is statistically dependent on the sequence. In order to account for this, we adopt an elegant 
technique by \citet{RS17} to control the resulting sequence of dependent random variables\footnote{This technique is 
commonly attributed to \citet{NJLS09}, but we find the connection with \citet{RS17} more illuminating. Otherwise, we 
learned this proof technique from \citet{NeuO24}.}. In particular, we introduce a second 
online learning algorithm for the sake of analysis, and use its regret bound to account for the additional error terms 
in the above decomposition. For sake of concreteness, we define the sequence of decisions made by this algorithm by 
setting $\wt{z}_1 = z_1$ and updating the parameters recursively via a mirror descent scheme analogous to the one 
underlying the sequence $z_k$:
\[
 \wt{z}_{k+1} = \argmin_{z\in\Zw} \ev{\iprod{g_k - \wt{g}_k}{z} + \frac{1}{\eta} \BB_\Psi(z\|z_k)}.
\]
Using this notation, the regret of the original algorithm can be rewritten as follows:
\begin{equation*}
\sum_{k=1}^{K} \iprod{g_k}{z_k - z^{*}} = \underbrace{\sum_{k=1}^{K}\siprod{\widetilde{g}_k}{z_k -
z^{*}}}_{R_K} + \underbrace{\sum_{k=1}^{K}\iprod{g_k - \widetilde{g}_k}{z_k - \widetilde{z}_k}}_{\wt{M}_K} +
\underbrace{\sum_{k=1}^{K}\siprod{g_k - \widetilde{g}_k}{\widetilde{z}_k - z^{*}}}_{\widetilde{R}_K}.
\end{equation*}
Thanks to this construction, the term $\wt{M}_K$ is a martingale and $\widetilde{R}_K$ is the regret of the auxiliarly 
online learning algorithm in the newly defined online learning game.

For the concrete proof of Lemma~\ref{lemma:OMD_bound}, we will make use of the following classic result regarding the 
regret of mirror descent.
\begin{lemma}
\label{lemma:OMD_bound2}
Let $ z \in \Zw $ and assume that $ \Psi $ is $ \lambda $-strongly convex with respect to some norm $  \norm{\cdot } $ 
whose dual is denoted by $\norm{\cdot}_*$. Consider the sequence with an arbitrary $u_1 \in \Zw$ and all subsequent 
iterates defined as
\[
 u_{k+1} = \argmin_{z\in\Zw} \ev{\iprod{v_k}{z} + \frac{1}{\omega} \BB_\Psi(z\|u_k)},
\]
where $a_k$ is an arbitrary sequence in $\Zw^*$ and $\omega > 0$. Then, for any $u^*\in\Zw$, the sequence $\pa{u_k}_k$ 
produced by OSMD satisfies the following bound:
\begin{equation}
\label{eq:OMD_bound3}
	\sum_{k=1}^{K} \siprod{a_k}{u_k - u} \leq \frac{\BB_{\Psi}(u^* \| u_1)}{\omega}+ \frac{\omega}{2 \lambda}
	\sum_{k=1}^{K}\norm{a_k}_*^2.
\end{equation}
\end{lemma}
The proof is standard and can be found in many textbooks---for concreteness, we refer to Theorem~6.10 of 
\citet{Ora19}. To proceed, we apply this lemma to the standard sequence of iterates in our setting with $a_k = 
\wt{g}_k$ and $\omega = \eta$ to bound $R_K$ and once again with $a_k = g_k - \wt{g}_k$ and $\omega = \wt{\eta}$ to 
bound $\wt{R}_K$. Finally, we use the Hoeffding--Azuma inequality (Lemma~\ref{lemma:Azuma_Hoeffding}) to control the 
remaining term as
\begin{align*}
\wt{M}_K = \sum_{k=1}^{K}\iprod{g_k - \widetilde{g}_k}{z_k - \widetilde{z}_k} \le 
C\sqrt{2\sum_{k=1}^{K}\norm{g_k - \widetilde{g}_k}^2 \log \frac 1\delta}
\end{align*}
with probability at least $1-\delta$. Indeed, notice that under the condition $\max_{z,z'\in\Zw} \norm{z-z'} \le C$ 
each term satisfies $\abs{\iprod{g_k - \widetilde{g}_k}{z_k - \widetilde{z}_k}} \le C \norm{g_k - \widetilde{g}_k}_*$, 
which allows using Lemma~\ref{lemma:Azuma_Hoeffding} with $c_k = 2 C \norm{g_k - \widetilde{g}_k}_*$. Putting these 
results together concludes the proof of Lemma~\ref{lemma:OMD_bound}.

\section{Technical Lemmas}\label{app:tech}
\begin{lemma}
\label{lemma:strong_convexity_entropy_cond}
The function $ \Psi(p) = \sum_{j=1}^{J} \sum_{i=1}^{I} p(i|j) \log p(i|j) $ is $1$-strongly convex with respect to the 
2-1 group norm $ \norm{p}_{2,1} = \sqrt{\sum_{j=1}^{J}\left( \sum_{i=1}^{I} \abs{p_{i|j}} \right)^2} $ on the set 
$\Zw = \{ p \in \real^{I\times J} \,:\, p(i|k) \ge 0 (\forall i,j),\, \sum_{i=1}^{I} 
p(i|j) = 1 \, (\forall j)$.
\end{lemma}
\begin{proof}
We first note that, by standard calculations, the Bregman divergence induced by $\Psi$ is
\[
 B_{\Psi}(p\|q) = \sum_{j=1}^J \sum_{i=1}^I p(i|j) \log \frac{p(i|j)}{q(i|j)}.
\]
Now, by Pinsker's inequality, we have that
\[
 B_{\Psi}(p\|q) \ge \frac 12 \sum_{j=1}^J \onenorm{p(\cdot|j) - q(\cdot|j)},
\]
which is equivalent to the statement of the lemma.	
\end{proof}

\begin{lemma}(Hoeffding--Azuma inequality, see, e.g., Lemma~A.7 in \citealt{CL06})
\label{lemma:Azuma_Hoeffding}
Let $ \pa{Z_k}_{k} $ be a maringale with respect to a filtration $ \pa{\mathcal{F}_k}_{k} 
$.
Assume that there are predictable processes $ \pa{A_k}_k $ and $ \pa{B_k}_k $ and positive constant $\pa{c_k}_k$ such 
that for all $k\geq 1 $, almost surely, 
\begin{equation*}
	A_k \leq Z_{k} - Z_{k-1} \leq B_k \quad \text{and} \quad B_k - A_k \leq c_t.
\end{equation*}
Then, for all $ \epsilon> 0 $,
\begin{equation}
\label{eq:Azuma_Hoeffding_eps}
	\PP{Z_t - Z_0 \geq \epsilon} \leq \exp{\left( -\frac{2 \epsilon^2}{\sum_{i=1}^{t}c_i^2} \right)},
\end{equation}
or equivalently for all $ \delta \in (0,1) $
\begin{equation}
\label{eq:Azuma_Hoeddfind_delta}
	\PP{Z_t - Z_0 \geq \sqrt{\frac{\left( \sum_{i=1}^{t}c_i^2 \right)\log(\frac{1}{\delta})}{2}}} \leq  \delta .
\end{equation}
\end{lemma}

\begin{lemma}\label{lemma:marginal_occupancy}
The occupancy measure $\nux\in\real_+^{\X\times\X}$ of the Markov chain $M_\X$ is uniquely defined by the two sets of equations
\begin{align}
\sum_{x'} \nux(x,x') &= \gamma \sum_{x''} \nux(x'',x) + (1-\gamma) \nu_{0,\X}(x) \quad (\forall x),\label{eq:nuxflow}\\
\nux(x,x') &= \PX(x'|x) \sum_{x''} \nux(x,x'') \quad (\forall x,x').\label{eq:nuxdef}
\end{align}
\end{lemma}

\begin{proof}
Using the definition of the occupancy measure $\nux$ we obtain
\begin{align*}
\nux(x,x') &= (1-\gamma) \sum_{t=0}^\infty \gamma^t \PP{X_t=x,X_{t+1}=x'}\\
 &= (1-\gamma) \sum_{t=0}^\infty \gamma^t \PX(x'|x) \PP{X_t=x}\\
 &= \PX(x'|x) \pa{ (1-\gamma) \nu_{0,\X}(x) + (1-\gamma) \sum_{t=1}^\infty \gamma^t \PP{X_t=x} }\\
 &= \PX(x'|x) \pa{ (1-\gamma) \nu_{0,\X}(x) + \gamma \sum_{x''} (1-\gamma) \sum_{t=1}^\infty \gamma^{t-1} \PP{X_{t-1}=x'',X_t=x} }\\
 &= \PX(x'|x) \pa{ (1-\gamma) \nu_{0,\X}(x) + \gamma \sum_{x''} \nux(x'',x) },
\end{align*}
where we used the stationarity of the transition kernel $\PX$, the definition of $\nu_{0,\X}$, the law of total probability,
and the stationarity of the Markov chain to recognize $\nux(x'',x)$ in the last step.
Summing the previous equation over $x'$ yields~\eqref{eq:nuxflow}, and substituting~\eqref{eq:nuxflow} into the previous equation yields~\eqref{eq:nuxdef}.

In order to show that the solution $\nu_\X$ to~\eqref{eq:nuxflow} and~\eqref{eq:nuxdef} is unique, we introduce the 
notation $\xi_\X$ as $\xi_\X(x)=\sum_{x'} \nux(x,x')$ for each $x$. Substituting~\eqref{eq:nuxdef} 
into~\eqref{eq:nuxflow} yields
\[
\xi_\X(x) = \gamma \sum_{x''} \PX(x|x'') \xi_\X(x'') + (1-\gamma) \nu_{0,\X}(x) \quad (\forall x).
\]
By defining $\xi_\X$ and $\nu_{0,\X}$ as vectors and $\PX$ as a matrix, we can write this system of equations in matrix 
form as $\xi_\X = \gamma \PX\xi_\X\transpose + (1-\gamma) \nu_{0,\X}$, or equivalently, $ (I-\gamma\PX\transpose) 
\xi_\X = (1-\gamma) \nu_{0,\X}$. Since $\PX$ is a positive matrix with spectral radius $1$, the Perron--Frobenius 
theorem applies and the matrix $(I-\gamma\PX)$ is invertible. Hence there exists a unique solution $\xi_\X = 
(1-\gamma) (I-\gamma\PX\transpose)^{-1} \nu_{0,\X}$, which together with~\eqref{eq:nuxdef} implies that $\nu_\X$ is 
uniquely defined as $\nu_\X(x,x')=\PX(x'|x)\xi_\X(x)$.
\end{proof}

\section{Additional details on experiments}\label{sec:additional_exp}

In this appendix we present further details about the experiments included in the main text, along with some 
additional empirical results. Along the way, we will also provide some further comments on best practices when 
implementing \SOMCOT, including recommended hyperparameter settings.

\subsection{Similarity metrics between parametric Markov chains}

This experiment serves to illustrate the ability of bisimulation metrics to capture intuitive similarities between 
stochastic processes, as well as the empirical behavior of \SOMCOT when used to approximate such similarity metrics 
based on data. To this end, we generated several random walk instances from the same family as used in the experiments 
in the main body (described in detail in Appendix~\ref{sec:implementation_details}). For this experiment, we let $\X = 
\Y = \ev{1,2,\dots,n}$ with $n=1$ and set the block size as $B=1$. Deviating from the setup described in 
Appendix~\ref{sec:implementation_details}, we set the reward function as $r(1) = r(n) = 1$ for both extremes of the 
state space, which induces a symmetry on the state space. For generating the set of environments, we varied the initial 
states $x_0$ and $y_0$ between $\ev{2,3,\dots,n-1}$ and the bias parameter $\theta$ in the set $\ev{0.05, 0.1, 0.15, 
\dots, 0.95}$, thus resulting in 72 different instances. We then computed pairwise distances between these instances 
using \SOMCOT with various sample sizes, and compared the results with the ground truth (computed by the Sinkhorn Value 
Iteration method of \citealt{Calo_J_N_S_S24}). Figure~\ref{fig:similarity} shows the similarity matrices obtained by 
these methods, showing that the distances computed by \SOMCOT successfully capture the structure of the problem: even 
though the exact numerical values of the true distances are not approximated very accurately, the qualitative picture 
obtained by \SOMCOT is very similar to the ground truth. In particular, the symmetry induced our choice of reward 
function is clearly visible with the matrix being symmetric along the counter-diagonal as well as the main diagonal.

\begin{figure}
\centering
 \includegraphics[width = .325\textwidth]{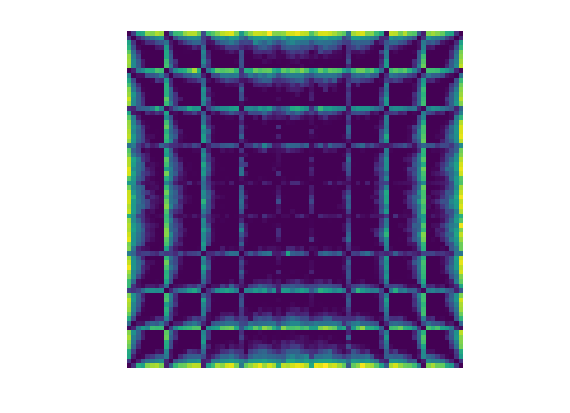}
 \includegraphics[width = .325\textwidth]{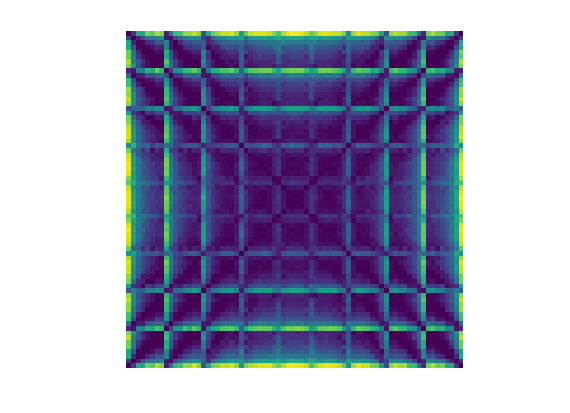}
 \includegraphics[width = .325\textwidth]{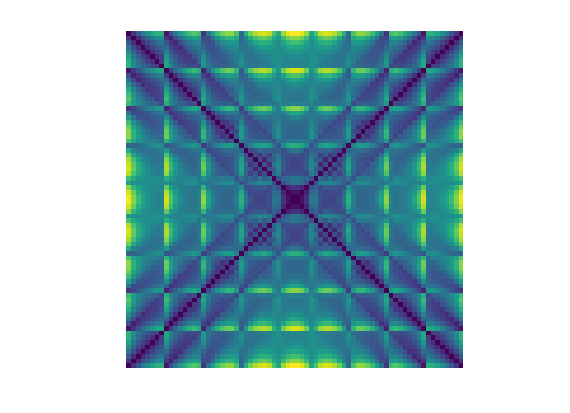}
 \caption{Distance matrices between instances after running \SOMCOT for $1000$ and $10000$ steps, and the ground 
truth obtained via Sinkhorn Value Iteration.}
\label{fig:similarity}
\end{figure}

\subsection{Practical implementation details}

Being a primal-dual method, \SOMCOT is not as easy to tune as a common stochastic optimization algorithm. There are 
several implementation details that one needs to design carefully in order to make sure that the algorithm behaves in a 
stable way and outputs good estimates. This section describes our experience working with \SOMCOT, and provides 
practical guidance for implementation.

% \paragraph{The learning rates.} 
\SOMCOT{} has one tunable parameter per optimization variable: a positive learning rate 
that controls the magnitude of the updates during optimization. While our theoretical analysis suggests some specific 
values for these learning rates to guarantee convergence, such values are typically too conservative (as is common in 
stochastic optimization). In practice, using larger learning rates can significantly reduce the number of iterations 
needed to reach good solutions. Since our problem involves optimizing six variables, this leads to six separate 
hyperparameters, which makes tuning a grueling task. To address this, we tie some of the learning rates 
together: all primal variables share a single learning rate denoted by $\eta$, and all dual variables share another one 
denoted by $\beta$.

Moreover, we observed that in practice using a fixed value for the dual learning rate $\eta$ often made it difficult to achieve stable convergence across different problem instances. 
To address this, we introduced a decaying learning-rate scheme of the form $\eta_k = \frac{\eta_0}{\sqrt{1 + ak}}$, 
where $k$ is the index of the current iteration and $a > 0$ is a tunable parameter. 
This decay helps balance the need for large updates in early iterations with the stability required for convergence in later stages.

Figure \ref{fig:etas} illustrates the performance of the algorithm under different learning rate settings. This shows 
that, even given the above choices, it is not easy to pick hyperparameters that work uniformly well across problem 
instances. Even for a single instance, the combination of $\eta_0$ and $\beta$ that leads to the best performance 
requires careful hyperparameter search. In order to understand the behavior of \SOMCOT under different parameter 
choice, it is helpful to remember the roles of the primal and dual variables, and in particular that the dual variables 
serve to penalize the primal variables for violating the primal constraints. Thus, a value of $\eta$ that is too high 
relative to $\beta$ leads to large constraint violations, resulting in $\mu$ values that yield very small 
distances but fall outside the feasible set. Ultimately, setting $\beta$ too small results in gross underestimation of 
the true distance. Thus, whenever one sees distance estimates that are suspiciously close to zero, the value of $\beta$ 
should be increased or the value of $\eta$ be decreased. The opposite scenario produces 
the inverse effect: a $\beta$ that is too large relative to $\eta$ causes the dual variables to update too quickly, 
leading to a resulting distance that overestimates the actual value. This issue can largely be mitigated by decaying 
$\eta$ while keeping $\beta$ constant (as described above). In our experience, it is often better to pick a large 
initial value for $\eta$: while this typically leads to a rapid drop of the distance estimate to zero, the estimates 
eventually start increasing and converge toward the true cost.

\begin{figure}
	\center 
	\includegraphics[width=0.9\textwidth]{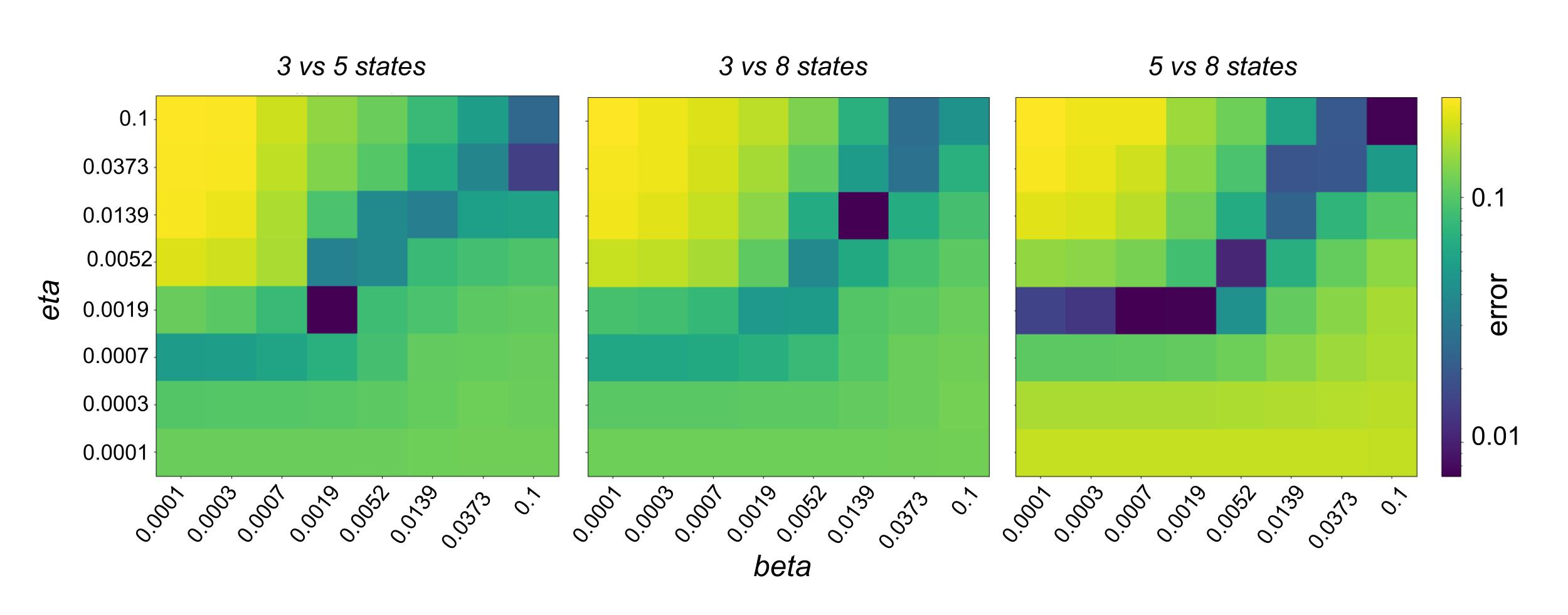} 
	\caption{The influence of the ratio between $\eta$ and $\beta$ on the convergence of \SOMCOT for different chain 
sizes. Error and learning rates are shown on a logarithmic scale. To produce this plot, a decay rate of $a=0.001$ was 
used for $\eta$. No decay was applied on $\beta$.}
	\label{fig:etas}
\end{figure}

Theorem~\ref{thm:PAC_bound_main} provides guarantees for the averaged output $ \bmu_K = 
\frac{1}{K} \sum_{k=1}^{K} \mu_k $, where $\mu_k$ is the value of $\mu$ obtained at iteration $k$. This is commonly 
required for algorithms based on regret analysis, at least for the theoretical guarantees to go through. In typical 
applications of stochastic optimization, this averaging step is not strictly necessary and the final iteration can 
perform well enough. However, this is typically not the case for primal-dual algorithms like \SOMCOT, where iterate 
averaging often makes a big difference to the stability of algorithms. This is true in our case too: without averaging, 
the iterates typically fluctuate quite wildly around the optimum. Averaging makes the estimates much more stable, and 
is thus strongly recommended (even if only for the last half of the iterates or less).

Finally, we note that all our experiments have made use of i.i.d.~transitions sampled from the occupancy measures of 
the two chains. This falls in line perfectly with the theory, but may be impractical in applications where transitions 
may be dependent or be sampled from undiscounted trajectory distributions. While we have not experimented with such 
data, we believe that \SOMCOT should be able to deal with it as long as efforts are made to break the correlations 
between the consecutive samples, for instance by sampling the transitions randomly from a buffer (instead of processing 
them in their original order). In our experiments, we have sometimes made use of minibatch updates, which can affect 
computational efficiency and stability, but no major impact on the overall convergence properties has been observed. We 
display all hyperparameter choices we have made in the experiments in Table~\ref{tab:params}.

\subsection{Details about the environments}\label{sec:implementation_details}
Our experiments made use of two families of Markov chains: a collection of parametrized random walks, and several 
instances of the classic ``inverted pendulum'' environment. We describe the details of these settings below.

\textbf{Parametrized random walks.}  We consider a one-dimensional random walk over a finite state space 
$\X = \ev{1,2,\dots,n}$ with biased transitions. A transition from state $x$ moves to $x+1$ with probability $\theta 
\in [0, 1]$ and $x-1$  with probability $1 - \theta$. States $1$ and $n$ are ``sticky walls'': the process remains 
there with probability $0.9$ or moves to the neighboring state with probability $0.1$. Additionally, we define a reward 
function on the 
state space, with values  $r(1) = 1$, $r(n) = -1$, and $r(x) = 0$ for all $x \in {2, 
\dots, n-1}$.  The initial state distribution is a Dirac measure on $x=1$.
To produce the plot shown in Figure \ref{fig:blockMC}, we generate a low-dimensional chain $\MX$ with bias $\theta=0.5$ following this setting. 
Then, we produce a set of chains $\MY \in \mathcal{B}$, each of them with a different bias parameter. In addition, all 
$\MY$ are augmented with an additional irrelevant noise variable, producing $B$ observations per each latent state in 
$\X$ . Formally, $\MY$ is a Markov chain on the state space $\Y$ equal to $\X\times\ev{1,2,\dots,B}$. 
The cost between $x\in\X$ and $y\in\Y$ is given by $c(x, y) = |r(x) - r(y)|$, reflecting the absolute difference in rewards between the states.

\textbf{Inverted pendulum.} We begin by training a near-optimal policy using DDPG, a widely known Deep RL algorithm, 
in the standard \textit{Pendulum-v1} environment, which is then used to induce a Markov chain. One could use any policy, 
but using a near-optimal policy produces richer dynamics (e.g., using a random policy in the \textit{Pendulum-v1} 
environment reduces the effective state space to the surroundings of the initial state). Once a policy is fixed, we 
discretized each state variable of the environment into $n$ bins. Since the \textit{Pendulum-v1} environment has 2 
variables (angle $\theta$ and angular velocity $w$), the resulting state space has $n^{2}$ states.  In our experiments, 
we have chosen $n=7$, which resulted in a total of $49$ states. Note that due to the discretization, the resulting 
stochastic process is no longer a Markov chain, as the states are no longer sufficient to predict the distribution of 
the next state. Nevertheless, the conducted experiments follow the same principle as the aforementioned random walks: 
We will compare the (approximate) Markov chain $\MX$ of discretized observations with a set of parametrized models 
$\MY$. The parameter governing the dynamics the acceleration constant $g$, capturing the effect of gravity. We set the 
default value $g=9.8$ in $\MX$, and choose values in $[0.4, 20]$ in the model set.
The cost function is given by $c(x, y) = |r(x) - r(y)|$, where $r(x)$ is the average reward in bin 
$x$ (computed from all samples that fell into bin $x$ along a long simulated trajectory).

\begin{table}
\center
 \begin{tabular}{c|c|c|c|c|c}
  Experiment & $\eta_0$ & $a$ & $\beta$ & $b$ & $\gamma$ \\
  \hline
  Figure~\ref{fig:encoder-decoder} & 40 & 0 & 0.2 & 1 & 0.99
  \\
  Figure~\ref{fig:similarity} & 20 & 0 & 0.5 & 1 & 0.99
  \\
  Figure~\ref{fig:blockMC} & 0.1 & 0.001 & 0.5 & 8 & 0.95
  \\
  Figure~\ref{fig:pendulum} & 0.1 & 0.05 & 0.2 & 16 & 0.95
 \end{tabular}
 \vspace{.5cm}
 \caption{Table summarizing our hyperparameter choices for each experiment. Recall that the learning rates follow the 
decaying scheme $\eta_k = \frac{\eta_0}{\sqrt{1 + ak}}$, and the minibatch size is denoted by $b$.}\label{tab:params}
\end{table}

\end{document}